\documentclass{article}

\usepackage{microtype}
\usepackage{graphicx}
\usepackage{subcaption}

\usepackage{booktabs}
\usepackage{array}
\usepackage{multirow}
\usepackage{tabularx}
\usepackage{makecell}
\usepackage{threeparttable}
\usepackage{adjustbox}

\graphicspath{{./}{figs/}{figs_datafm/}{figs/figs_datafm/}}
\usepackage{caption}

\usepackage{xcolor}
\usepackage{tcolorbox}
\tcbuselibrary{skins,breakable,theorems,listings}
\usepackage{listings}

\usepackage{tikz}
\usetikzlibrary{positioning,arrows.meta,shapes.misc,fit,calc}

\usepackage{amsmath,amssymb,mathtools,amsthm}
\usepackage{pifont}
\usepackage{fontawesome5}

\usepackage{enumitem}
\usepackage{xspace}

\definecolor{thmBlue}{RGB}{26,45,90}
\definecolor{thmGold}{RGB}{180,130,20}
\definecolor{thmGreen}{RGB}{30,100,60}
\definecolor{thmRed}{RGB}{160,30,30}
\definecolor{defBg}{RGB}{245,248,255}
\definecolor{thmBg}{RGB}{255,252,240}
\definecolor{propBg}{RGB}{245,255,248}
\definecolor{remBg}{RGB}{252,245,245}
\definecolor{promptBg}{RGB}{245,247,252}
\definecolor{promptFrame}{RGB}{26,45,90}
\definecolor{systemBg}{RGB}{255,250,235}
\definecolor{systemFrame}{RGB}{160,110,20}
\definecolor{codeBg}{RGB}{248,248,248}
\definecolor{fillSlot}{RGB}{180,30,30}
\definecolor{warnOrange}{RGB}{200,90,10}
\definecolor{remGray}{RGB}{90,90,90}

\theoremstyle{plain}
\newtheorem{theorem}{Theorem}[section]

\newtheorem{corollary}[theorem]{Corollary}
\newtheorem{proposition}[theorem]{Proposition}
\newtheorem{axiom}[theorem]{Axiom}

\theoremstyle{definition}
\newtheorem{definition}{Definition}[section]

\theoremstyle{remark}
\newtheorem{remark}{Remark}[section]

\tcbset{
  icmlthm/.style={
    enhanced, breakable, boxrule=0.7pt, arc=3pt,
    left=8pt, right=8pt, top=5pt, bottom=5pt,
    fonttitle=\bfseries\small,
  },
  sysprompt/.style={
    enhanced, breakable,
    colframe=systemFrame, colback=systemBg,
    boxrule=0.7pt, arc=3pt,
    left=8pt, right=8pt, top=4pt, bottom=4pt,
    fonttitle=\bfseries\small, coltitle=white,
    attach boxed title to top left={yshift=-2mm, xshift=6pt},
    boxed title style={colback=systemFrame, arc=2pt, boxrule=0pt},
  },
  userprompt/.style={
    enhanced, breakable,
    colframe=promptFrame, colback=promptBg,
    boxrule=0.7pt, arc=3pt,
    left=8pt, right=8pt, top=4pt, bottom=4pt,
    fonttitle=\bfseries\small, coltitle=white,
    attach boxed title to top left={yshift=-2mm, xshift=6pt},
    boxed title style={colback=promptFrame, arc=2pt, boxrule=0pt},
  },
  thmbox/.style={
    enhanced, colframe=thmBlue, colback=defBg!60, boxrule=0.7pt, arc=3pt,
  },
  defbox/.style={
    enhanced, colframe=thmGreen, colback=defBg!40, boxrule=0.6pt, arc=3pt,
  },
  warnbox/.style={
    enhanced, colframe=warnOrange, colback=warnOrange!8, boxrule=0.6pt, arc=3pt,
  },
}

\tcolorboxenvironment{proof}{
  blanker, breakable, left=6pt, right=0pt,
  before skip=2pt, after skip=6pt,
  borderline west={1.2pt}{0pt}{thmGold!60},
}

\lstdefinestyle{promptstyle}{
  basicstyle=\ttfamily\footnotesize,
  backgroundcolor=\color{codeBg},
  breaklines=true, breakatwhitespace=true,
  columns=fullflexible, keepspaces=true,
  showstringspaces=false,
  moredelim=[is][\color{fillSlot}\bfseries]{<}{>},
  frame=none,
}

\lstdefinestyle{pseudocode}{
  basicstyle=\ttfamily\small,
  backgroundcolor=\color{codeBg},
  breaklines=true, keepspaces=true,
  showstringspaces=false,
  commentstyle=\color{remGray}\itshape,
  morecomment=[l]{\#},
  frame=none,
}

\newcommand{\E}[1]{\mathbb{E}\!\left[#1\right]}
\newcommand{\prob}[1]{p\!\left(#1\right)}
\newcommand{\vneg}{v_{\mathrm{neg}}}
\newcommand{\commit}{c}
\newcommand{\Acccov}{\mathrm{Acc}_{\mathrm{cov}}}
\newcommand{\ECEcov}{\mathrm{ECE}_{\mathcal{C}}}
\newcommand{\Cov}{\mathrm{Cov}}
\newcommand{\Yes}{\textsc{yes}}
\newcommand{\No}{\textsc{no}}

\newcommand{\cmark}{\ding{51}}
\newcommand{\xmark}{\ding{55}}
\newcommand{\pyes}{\ding{51}}

\usepackage{hyperref}

\usepackage[accepted]{icml2026_fogen}

\usepackage[capitalize,noabbrev]{cleveref}

\usepackage[textsize=tiny]{todonotes}

\icmltitlerunning{Commitment-Aware Logical Reasoning Evaluation in Deep Generative Models}

\begin{document}

\twocolumn[
  \icmltitle{Coherence Under Commitment: Probing Generalization and Vacuous Memorization in LLM Logical Reasoning}



\icmlsetsymbol{equal}{*}

\begin{icmlauthorlist}
    \icmlauthor{Noor Islam S. Mohammad}{equal,yyy}
    \icmlauthor{Mahmudul Hasan}{comp}
\end{icmlauthorlist}

\icmlaffiliation{yyy}{Department of Computer Science, Informatics Institute, Istanbul Technical University, İstanbul, Türkiye}
\icmlaffiliation{comp}{School of Information Technology, Deakin University, Geelong, VIC 3220, Australia}

\icmlcorrespondingauthor{Noor Islam S. Mohammad}{islam23@itu.edu.tr}

\icmlkeywords{Large Language Models, Logical Reasoning, Evaluation Methodology, Calibration, Uncertainty Estimation}

  \vskip 0.3in ]

\printAffiliationsAndNotice{}  

\begin{abstract}
Large language models (LLMs) deployed for logical reasoning in knowledge-intensive domains exhibit a subtle but critical failure: \emph{coherence can be vacuously achieved through systematic abstention}. A model that withholds commitment to either entailment or refutation satisfies negation consistency while providing no utility. We introduce \textbf{Coherence Under Commitment (CUC)}, a dual-query evaluation paradigm that jointly measures consistency and decisiveness. CUC contributes three innovations: (1)~a \textbf{commitment score} $c(\varphi) = p(\varphi) + p(\lnot\varphi)$ quantifying probability mass allocated to decisive outcomes; (2)~a \textbf{deterministic elicitation protocol} via normalized YES/NO log probabilities, eliminating sampling variance; and (3)~a \textbf{3-way decision framework} (\textsc{True}/\textsc{False}/\textsc{Uncertain}) operationalizing the coherence-commitment trade-off into metrics. Experiments on four open-weight LLMs (1B--3B) across 204 \textsc{FOLIO} examples expose a sharp frontier. Qwen2.5-3B achieves near-zero contradiction ($\mathbb{E}[v_{\mathrm{neg}}]{=}0.025$) but only $7.4\%$ coverage, while TinyLlama-1.1B reaches $79.4\%$ coverage with violations on every example. Coherence-only evaluation would rank the abstaining model first—CUC exposes this as vacuous, and the frontier generalizes to \textsc{LogiQA}~v2 ($\rho{=}0.97$). We argue that evaluation must report both coherence \emph{and} non-vacuous commitment and release a toolkit for standardized assessment. {\faGithub\ Code and data available at \url{https://pmlrbd.github.io/auc.ml/}} 
\end{abstract} 

\section{Introduction}

Large language models have achieved remarkable performance on reasoning benchmarks, driven by prompting innovations such as chain-of-thought \citep{wei2022cot}, zero-shot reasoning \citep{kojima2022zeroshot}, and self-consistency decoding \citep{wange2022selfconsistency}. However, a growing body of evidence reveals that surface-level accuracy obscures deeper reliability failures. Models exhibit overconfidence \citep{guo2017calibration}, generate unfaithful rationales \citep{turpin2023unfaithful}, and produce hallucinations that undermine trust \citep{ji2023hallucination,lin2022truthfulqa}. These failures are especially consequential in knowledge-intensive multimodal settings—clinical diagnostics, scientific literature reasoning, and embodied planning—where models must base decisive judgments on heterogeneous evidence (images, sensor streams, structured records, and text) and where abstaining from a conclusion carries real operational cost.

Specifically, inconsistency in logical reasoning poses unique risks in such domains. When a system alternates between endorsing $\varphi$ and not endorsing
$\neg\varphi$ under identical premises $P$, it renders downstream pipelines fundamentally unreliable: a radiology VQA model that contradicts itself across
semantically equivalent image-report queries cannot be clinically trusted; a robotic planner that affirms mutually exclusive action preconditions provides no actionable guidance. The stakes are particularly high wherever logical consistency is not merely desirable but necessary for valid, auditable inference chains.

\paragraph{The coherence evaluation trap.}
A natural response is to evaluate whether LLMs respect basic logical axioms. Negation consistency—the requirement that a model should not simultaneously affirm $P \models \varphi$ and $P \models \neg\varphi$---is a minimal desideratum. Yet we identify a critical failure mode: \textbf{coherence can be vacuously achieved through systematic abstention}. A model that refuses to commit to either entailment or refutation trivially satisfies negation consistency while providing zero reasoning utility. In knowledge-intensive deployment, this failure mode is invisible to standard evaluation: a domain-specialized system that declines to affirm or deny a grounded conclusion appears "coherent" while being wholly unfit for purpose.

This insight has profound implications for evaluation methodology. Prior coherence-centric evaluations may inadvertently reward models that "pass" by declining to answer, creating a false impression of logical reliability. We demonstrate that this is not a hypothetical concern: among four models we evaluate, the most "coherent" achieves its low contradiction rate \emph{primarily through abstention}. This finding directly challenges the prevailing assumption that low contradiction rates indicate strong logical reasoning capability—and motivates a new evaluative standard for knowledge-intensive multimodal systems.

\paragraph{Our contributions.}
We propose \textbf{Coherence Under Commitment}, a unified evaluation paradigm that addresses this blind spot through four interconnected innovations: (i). \textbf{The Commitment Score (Novel Metric):} We introduce $c(\varphi) = \pyes(\varphi) + \pyes(\neg\varphi)$, measuring the total probability mass allocated to decisive outcomes. When $c(\varphi) \ll 1$, the model treats the query as effectively unknown, regardless of apparent coherence. This is the \emph{first} metric specifically designed to detect vacuous coherence in logical evaluation, and it applies without modification to any system—text-only or multimodal—where complementary query pairs $(\varphi,\neg\varphi)$ can be posed over grounded premises. (ii). \textbf{Deterministic Black-Box Elicitation (Novel Protocol):} We develop a reproducible protocol using normalized log-probabilities over YES/NO responses that eliminates sampling variance while maintaining model-agnostic applicability. Unlike sampling-based consistency checks \citep{manakul2023selfcheckgpt}, our approach yields identical results across runs and requires only two forward passes per example.

(iii) \textbf{The Coherence-Under-Commitment Frontier (Novel Framework):} We formalize the empirical observation that coherence and commitment trade off against each other into a measurable frontier. This provides the first principled framework for comparing models that optimize different points on this trade-off, preventing misleading comparisons that favor vacuously coherent models—a risk that is amplified in domain-specialized benchmarks where abstention can masquerade as calibrated uncertainty. (iv). \textbf{Comprehensive Ablation Analysis (Novel Methodology):} We provide the first systematic ablation study of commitment-aware evaluation, quantifying sensitivity to threshold selection ($\tau$, $\delta$), elicitation format (YES/NO vs.\ True/False), model scale, and architectural choices, establishing methodological best practices transferable to multimodal and domain-specialized evaluation regimes. Our experiments on the \textsc{FOLIO} benchmark reveal that existing coherence-only evaluation would rank a systematically abstaining model (Qwen2.5-3B) as "best," despite its 7.4\% coverage. Our commitment-aware framework exposes this as vacuous coherence, enabling meaningful model comparison and providing actionable guidance for practitioners building and evaluating reasoning systems in knowledge-intensive multimodal applications. We further confirm this frontier holds on \textsc{LogiQA}~v2, with near-perfect rank agreement ($\rho=0.97$) across benchmarks.

\section{Related Work}

\paragraph{LLM reasoning and benchmarks.}
Progress in reasoning has been catalyzed by prompting methods \citep{wei2022cot,kojima2022zeroshot,wange2022selfconsistency} and benchmark suites, including BIG-bench \citep{srivastava2022bigbench}. Logical reasoning datasets span multiple paradigms: \textsc{LogiQA} \citep{liu2020logiqa} tests natural language inference with 8,678 questions derived from logical examinations; the RuleTaker family \citep{clark2020sofreasoners} evaluates rule-following with synthetic logical theories; \textsc{ProofWriter} \citep{tafjord2021proofwriter} requires explicit proof generation, and \textsc{FOLIO} \citep{han2024folio} targets first-order logic with formal verification support. Multimodal reasoning benchmarks—MMMU, MedVQA, and ScienceQA—extend these challenges to heterogeneous input modalities but share the same evaluative blind spot we identify: accuracy on committed examples reveals nothing about the \emph{coverage} of those commitments. \textbf{Our work is orthogonal}: we contribute an evaluation \emph{methodology} applicable across all such benchmarks, not a new dataset.

\paragraph{Consistency and hallucination.}
Contradictions and hallucinations have been studied extensively \citep{ji2023hallucination}, with sampling-based consistency checks emerging as
detection tools \citep{manakul2023selfcheckgpt}. SelfCheckGPT detects hallucinations by measuring consistency across multiple sampled responses, while
our approach probes consistency across \emph{logically related queries} within a single deterministic evaluation. Chain-of-thought rationales can be unfaithful \citep{turpin2023unfaithful}, motivating evaluation methods independent of natural-language explanations. \textbf{Our distinction}: while prior work detects inconsistency \emph{within} generated outputs through sampling variance, we measure inconsistency \emph{across} complementary logical queries ($\varphi$ and $\neg\varphi$), directly probing the model's belief structure without relying on potentially unfaithful verbalizations—a property that carries over directly to multimodal settings where rationale faithfulness is even harder to verify.

\paragraph{Calibration and uncertainty.}
Calibration diagnostics, including Expected Calibration Error (ECE) \citep{naeini2015bbq,guo2017calibration}, quantify confidence-accuracy alignment. For LLMs, confidence elicitation has shown that probability estimates can be meaningful when properly extracted \citep{kadavath2022knowwhatknow,zhang2024calibratingconfidence}. Recent work has explored verbalized confidence \citep{lin2022truthfulqa} and
self-evaluation capabilities. \textbf{Our novelty}: we identify a logic-specific failure mode---\emph{vacuous coherence via abstention}---that is invisible to standard calibration metrics. A model can be perfectly calibrated on the examples where it commits while being systematically unhelpful on the
majority where it abstains; standard ECE would not flag this, as it measures calibration only on examples that receive a prediction. This gap is
particularly dangerous in knowledge-intensive domains where practitioner trust is built on the \emph{completeness} of a system's judgments, not merely their
per-prediction accuracy.

\paragraph{Evaluation in multimodal and domain-specialized settings.}
Evaluation of multimodal reasoning systems has largely relied on accuracy-centric benchmarks that inherit the same blind spot. In knowledge-intensive domains—radiology report grounding, scientific claim verification, and robot task feasibility assessment—abstention is not a neutral outcome; it represents a failure to leverage available domain evidence. CUC extends naturally to these settings: any two complementary queries ($\varphi$: "Does image $I$ support diagnosis $D$?" and $\neg\varphi$: "Does image $I$ contradict diagnosis $D$?") instantiate the two-query protocol without modification. The commitment score $c(\varphi)$ and deterministic log-probability elicitation (Section~\ref{sec:elicitation}) are modality-agnostic; only the premise representation changes.

\paragraph{Positioning our contribution.}
Existing work treats coherence and calibration as independent concerns measured separately. We unify them through the coherence-commitment frontier, demonstrating they are fundamentally coupled in logical evaluation—a coupling that becomes \emph{more} consequential, not less, as reasoning systems are
deployed in specialized multimodal domains. This represents a paradigm shift: from asking "Is the model coherent?" to asking "Is the model \emph{usefully}
coherent?" Our framework provides the first unified lens for understanding these trade-offs across text-only and multimodal knowledge-intensive evaluation.

\section{Methodology}
\subsection{Problem Setup}

Large language models (LLMs) are increasingly deployed for reasoning tasks in knowledge-intensive domains, including scientific inference, clinical decision support, and multimodal understanding. Despite their strong performance on many benchmarks, these systems exhibit a subtle but critical failure mode: outputs may appear logically coherent while failing to reflect meaningful epistemic commitment. In practice, models may avoid contradiction not by correct reasoning but by abstaining from taking a position altogether. This creates a gap between \emph{apparent coherence} and \emph{actual reasoning utility}, which becomes especially problematic in high-stakes environments where decisions must be decisive and grounded. We address this issue by introducing a dual-query evaluation paradigm that jointly measures logical coherence and epistemic commitment. Our central hypothesis is that these two dimensions are not independent: improving one often degrades the other, forming a measurable trade-off that we term the \textbf{coherence–commitment frontier}.

\paragraph{Problem formulation.}
Let $P$ denote a set of natural-language premises and $\varphi$ a candidate conclusion. We treat an LLM as a probabilistic belief function over entailment judgments conditioned on $(P,\varphi)$. The goal of evaluation is not merely to determine whether a model avoids contradiction but whether it provides actionable reasoning by committing to a conclusion when warranted. While we instantiate our framework on textual reasoning tasks, the formulation is fundamentally modality-agnostic. The premises $P$ may represent structured inputs such as image–text pairs, sensor streams, scientific graphs, or other multimodal contexts. The only requirement is that we can define complementary entailment queries over $(P,\varphi)$ and $(P,\neg\varphi)$ in a consistent manner.

\paragraph{Dual-query evaluation paradigm.}
We construct a symmetric query pair:
\begin{align}
Q_{\varphi}:&\quad \text{Is } \varphi \text{ logically entailed by } P? \\
Q_{\neg\varphi}:&\quad \text{Is } \neg\varphi \text{ logically entailed by } P?
\end{align}

This formulation plays a dual role. First, it exposes logical inconsistencies by testing whether a model can simultaneously support contradictory hypotheses. Second, it reveals epistemic behavior: whether the model actively commits to one side or instead distributes probability mass in a way that avoids decision-making. We constrain responses to binary outputs (YES/NO) and compute their probabilities using normalized log-likelihoods. Let $\pyes(\varphi)$ and $\pyes(\neg\varphi)$ denote the probabilities assigned to affirmative responses for each query. This design provides a direct view of the model's internal belief distribution while avoiding sampling noise or decoding artifacts. Importantly, this formulation converts reasoning evaluation into a structured probabilistic comparison problem over complementary hypotheses, rather than a single-output classification task.

\paragraph{Coherence via negation violation.}
A rational reasoning system should not simultaneously endorse both a statement and its negation under identical premises. Violations of this principle indicate logical inconsistency in probability allocation. We quantify this using the negation-coherence violation:

\begin{equation}
v_{\text{neg}}(\varphi)=\max\bigl(0,\pyes(\varphi)+\pyes(\neg\varphi)-1\bigr).
\end{equation}

This metric measures the extent to which a model assigns more than unit probability mass across mutually exclusive outcomes. When $v_{\text{neg}}(\varphi)>0$, the model is effectively overcommitting across contradictory hypotheses; it reveals an incoherent belief structure. However, a key limitation emerges: coherence alone is insufficient as an evaluation signal. A model can trivially achieve zero violation by avoiding commitment entirely, assigning low probability to both outcomes.

\paragraph{Commitment score and epistemic utility.}
To address this limitation, we introduce the \textbf{commitment score}:
\begin{equation}
c(\varphi)=\pyes(\varphi)+\pyes(\neg\varphi).
\end{equation}

This quantity measures how much probability mass the model assigns to decisive answers. Unlike standard confidence measures, this score captures whether the model is willing to take a stance at all.

A crucial insight follows from the relationship:
\[
v_{\text{neg}}(\varphi)=\max(0,c(\varphi)-1).
\]

This implies that any model with $c(\varphi)\le1$ will necessarily achieve zero negation violations, regardless of its reasoning quality. In other words, a model can appear perfectly coherent while being epistemically non-committal. This leads to a fundamental evaluation failure in coherence-only metrics: models that systematically abstain can be incorrectly ranked as highly reliable, despite offering no actionable reasoning signal. This phenomenon becomes especially problematic in knowledge-intensive domains where abstention is not neutral—it represents a failure to provide a decision. We therefore interpret commitment as a measure of epistemic utility: higher commitment indicates a stronger willingness to resolve uncertainty into a concrete judgment.

\paragraph{3-way decision framework.}
To translate probabilistic outputs into actionable evaluation metrics, we define a structured 3-way decision rule:

\[
\widehat{y}(\varphi)=
\begin{cases}
\textsc{True} & \text{if }\pyes(\varphi)\ge\tau,\ \pyes(\varphi)\ge\pyes(\neg\varphi)+\delta,\\
\textsc{False} & \text{if }\pyes(\neg\varphi)\ge\tau,\ \pyes(\neg\varphi)\ge\pyes(\varphi)+\delta,\\
\textsc{Uncertain} & \text{otherwise.}
\end{cases}
\]

Here, $\tau$ defines the minimum confidence required to commit to a decision, while $\delta$ enforces a margin ensuring that predictions are not made under near-tied probabilities. This formulation explicitly separates three behavioral modes: (i) confident correctness, (ii) confident incorrectness, and (iii) epistemic abstention. Unlike binary evaluation metrics, this decomposition allows us to distinguish between models that are precise, incorrect, or systematically non-committal. We report three complementary metrics: overall accuracy, coverage (fraction of non-\textsc{Uncertain} predictions), and conditional accuracy on committed predictions (Acc$_{\text{cov}}$). This decomposition is essential for understanding whether performance gains arise from better reasoning or from selective abstention.

\section{Probability Elicitation Protocol}
\label{sec:elicitation}

A central challenge in evaluating LLM reasoning lies in extracting reliable probabilistic judgments from autoregressive models. Direct generation introduces stochasticity and conflates reasoning ability with decoding strategy. To address this, we adopt a deterministic elicitation approach based on normalized log-probabilities.

We compute:
\begin{equation}
\pyes(x)=\frac{e^{\log P(\textsc{yes}\mid x)}}
{e^{\log P(\textsc{yes}\mid x)}+e^{\log P(\textsc{no}\mid x)}}.
\end{equation}

This formulation converts token-level likelihoods into calibrated binary probabilities. For multi-token responses, we compute sequence likelihoods by summing token log probabilities:
\[
\log P(\textsc{yes}\mid x)=\sum_t \log P(\text{token}_t \mid x,\text{context}).
\]

This ensures that probability estimates reflect the full sequence likelihood rather than partial token predictions. Our approach aligns with prior work showing that LLM logits can serve as meaningful uncertainty estimates under controlled prompting regimes. However, unlike sampling-based methods, our procedure is fully deterministic, eliminating variance across runs.

\paragraph{Advantages of deterministic elicitation.}
Our protocol offers three key advantages over generation-based evaluation methods. First, it is fully reproducible: identical inputs produce identical outputs, ensuring consistent benchmarking across runs, models, and environments. Second, it avoids reliance on free-form rationales, which are known to be frequently misaligned with internal model reasoning processes. Third, it directly probes the model’s predictive distribution rather than its sampled outputs, yielding a more faithful representation of model belief structure. Computationally, the method is efficient, requiring only two forward passes per example corresponding to $Q_\varphi$ and $Q_{\neg\varphi}$. Importantly, it does not require architectural modification and is therefore applicable to both text-only and multimodal models, provided that log probabilities over constrained tokens are accessible.

\section{Experimental Setup}

\paragraph{Dataset.}
We evaluate on the \textsc{FOLIO} v0.0 validation set, consisting of 204 logically structured examples with gold 3-way labels: \textsc{True}, \textsc{False}, and \textsc{Uncertain}. Each instance contains premises $P$, a hypothesis $\varphi$, and a labeled outcome. Crucially, \textsc{FOLIO} provides dataset-grounded negations $\neg\varphi$, enabling clean instantiation of the dual-query framework without the need for heuristic negation construction. This makes the dataset particularly suitable for evaluating coherence and commitment separately, as it naturally encodes logical complementarity. To assess generalization beyond a single benchmark, we additionally evaluate on the \textsc{LogiQA}~v2 test split (304 examples) \citep{liu2023logiqa2}, constructing $\neg\varphi$ via rule-based linguistic negation since this dataset does not provide formally verified negation fields; results appear in Table~\ref{tab:generalization}.

\paragraph{Models.}
We evaluate four open-weight models spanning 1B–3B parameters: TinyLlama-1.1B-Chat, Qwen2.5-1.5B-Instruct, Qwen2.5-3B-Instruct, and Phi-2. This selection enables controlled analysis across model scale, training data composition, and instruction-tuning strategies.  Together, these models allow us to study how architectural scale and training regimes influence positioning on the coherence–commitment frontier. \textbf{Prompting protocol:} We use a fixed prompting template enforcing single-token YES/NO responses. This constraint ensures compatibility with log-probability extraction and removes variability introduced by free-form generation. Full prompts are provided in the appendix.

\paragraph{Evaluation protocol and metrics.}
We set $(\tau,\delta)=(0.60,0.10)$ as the default thresholds and report results under additional ablations. Metrics include mean commitment, mean negation violation, violation rate, coverage, and accuracy. We compute 95\% confidence intervals using bootstrap resampling with 1{,}000 samples (seed = 42). Across all experiments, we consistently observe a strong coherence–commitment trade-off: models with lower contradiction rates tend to exhibit lower commitment, while models that increase coverage often incur higher logical violations. This empirical pattern motivates our central thesis that reasoning evaluation must jointly account for both coherence and epistemic decisiveness.

\section{Results and Analysis}
\label{sec:results}

We evaluate four small open-weight models on 204 \textsc{FOLIO} validation examples using the CUC framework.  The primary metric table appears as Table~\ref{tab:main}; the companion scatter plots are in Figure~\ref{fig:pareto_grid}; aggregate bar charts and calibration curves appear in Figures~\ref{fig:model_bars}~and~\ref{fig:calibration_grid}. All ablation tables are co-located with their companion discussion in Section~\ref{sec:ablations}.  Three findings emerge with statistical reliability across every experimental condition: (i)~No model simultaneously achieves high coverage and low coherence violation; a genuine frontier exists. (ii)~The two models that \emph{appear} best under single-metric evaluation
owe their scores to opposite failure modes—abstention and overcommitment—not to sound reasoning. (iii)~Calibration on committed predictions is poor across the board, ruling out confident deployment at either extreme of the frontier.

\begin{table}[ht]
\centering
\footnotesize
\setlength{\tabcolsep}{3.5pt}
\renewcommand{\arraystretch}{1.12}
\caption{\textbf{Coherence Under Commitment on \textsc{FOLIO} v0.0 validation.} Accuracy scores differ by at most 0.098 across models, yet commitment and violation scores span an order of magnitude—demonstrating that accuracy alone cannot distinguish between qualitatively different reasoning behaviors. Acc: overall 3-way accuracy. Cov: coverage (fraction predicted True/False).  Acc $_{\text{cov}}$: accuracy on covered examples. $\E[c]$: mean commitment. $\E[v_{\text{neg}}]$: mean negation-coherence violation. \%$v_{\text{neg}}{>}0$: fraction with any violation. Brackets: 95\,\% bootstrap CIs ($B{=}1{,}000$).}
\label{tab:main}
\begin{threeparttable}
\begin{adjustbox}{max width=\linewidth}
\begin{tabular}{l c c c c c c c}
\toprule
\textbf{Model} & \textbf{$n$} &
  \makecell{\textbf{Acc}$\uparrow$} &
  \makecell{\textbf{Cov}$\uparrow$} &
  \makecell{\textbf{Acc$_{\text{cov}}$}$\uparrow$} &
  \makecell{\textbf{$\E[c]$}$\uparrow$} &
  \makecell{\textbf{$\E[v_{\text{neg}}]$}$\downarrow$} &
  \makecell{\textbf{\%$v_{\text{neg}}{>}0$}$\downarrow$} \\
\midrule
Phi-2 & 204 &
  \makecell{0.441\\{\scriptsize [0.373,\,0.510]}} &
  \makecell{0.417\\{\scriptsize [0.348,\,0.480]}} &
  \makecell{0.565\\{\scriptsize [0.458,\,0.667]}} &
  \makecell{1.164\\{\scriptsize [1.136,\,1.190]}} &
  \makecell{0.195\\{\scriptsize [0.175,\,0.215]}} &
  \makecell{0.789\\{\scriptsize [0.735,\,0.843]}} \\[2pt]
Qwen2.5-1.5B & 204 &
  \makecell{0.402\\{\scriptsize [0.338,\,0.471]}} &
  \makecell{0.309\\{\scriptsize [0.250,\,0.373]}} &
  \makecell{0.508\\{\scriptsize [0.386,\,0.629]}} &
  \makecell{0.674\\{\scriptsize [0.590,\,0.762]}} &
  \makecell{0.166\\{\scriptsize [0.129,\,0.205]}} &
  \makecell{0.461\\{\scriptsize [0.392,\,0.530]}} \\[2pt]
Qwen2.5-3B & 204 &
  \makecell{0.382\\{\scriptsize [0.319,\,0.451]}} &
  \makecell{0.074\\{\scriptsize [0.039,\,0.108]}} &
  \makecell{0.800\\{\scriptsize [0.571,\,1.000]}} &
  \makecell{0.115\\{\scriptsize [0.067,\,0.167]}} &
  \makecell{0.025\\{\scriptsize [0.010,\,0.044]}} &
  \makecell{0.064\\{\scriptsize [0.029,\,0.103]}} \\[2pt]
TinyLlama-1.1B & 204 &
  \makecell{0.343\\{\scriptsize [0.279,\,0.412]}} &
  \makecell{0.794\\{\scriptsize [0.735,\,0.848]}} &
  \makecell{0.346\\{\scriptsize [0.275,\,0.422]}} &
  \makecell{1.698\\{\scriptsize [1.687,\,1.709]}} &
  \makecell{0.698\\{\scriptsize [0.687,\,0.709]}} &
  \makecell{1.000\\{\scriptsize [1.000,\,1.000]}} \\
\bottomrule
\end{tabular}
\end{adjustbox}
\end{threeparttable}
\end{table}

\begin{figure}[tp]
\centering
\includegraphics[width=\linewidth]{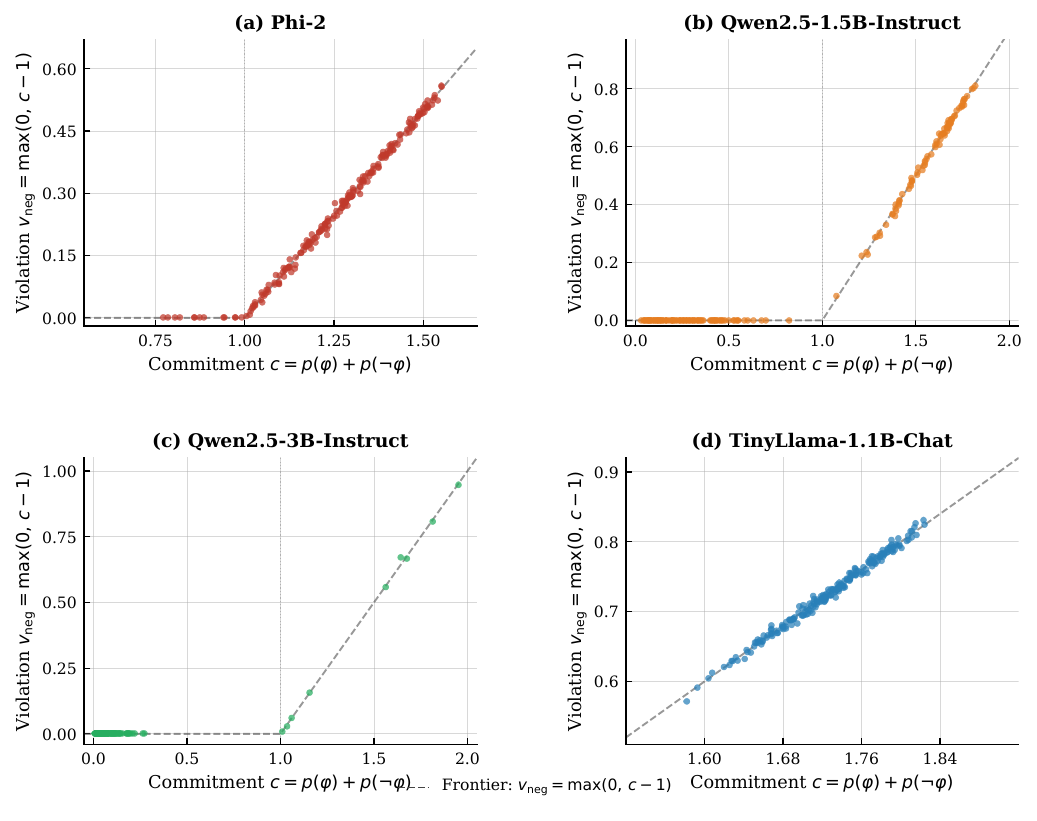}
\caption{\textbf{Per-example coherence-commitment frontier.} Each point represents one \textsc{FOLIO} validation example; the $x$-axis is the commitment score $c(\varphi)=p(\varphi)+p(\neg\varphi)$ and the $y$-axis is the negation-coherence violation $v_{\mathrm{neg}}=\max(0,\,c-1)$. The dashed line marks the theoretical frontier $v_{\mathrm{neg}}=c-1$, confirming the algebraic identity: low commitment mechanically guarantees low violation regardless of reasoning quality. Clustering patterns reveal distinct failure modes: Qwen2.5-3B concentrates near the origin (systematic abstention, $\bar{c}=0.115$); TinyLlama-1.1B saturates at high $c$ and high violation ($\bar{v}_{\mathrm{neg}}=0.698$, 100\% of examples violating); Phi-2 and Qwen2.5-1.5B occupy intermediate frontier positions.  A coherence-only evaluation would incorrectly rank Qwen2.5-3B as best.}
\label{fig:pareto_grid}
\end{figure}%


\begin{figure}[tp]
\centering
\includegraphics[width=\linewidth]{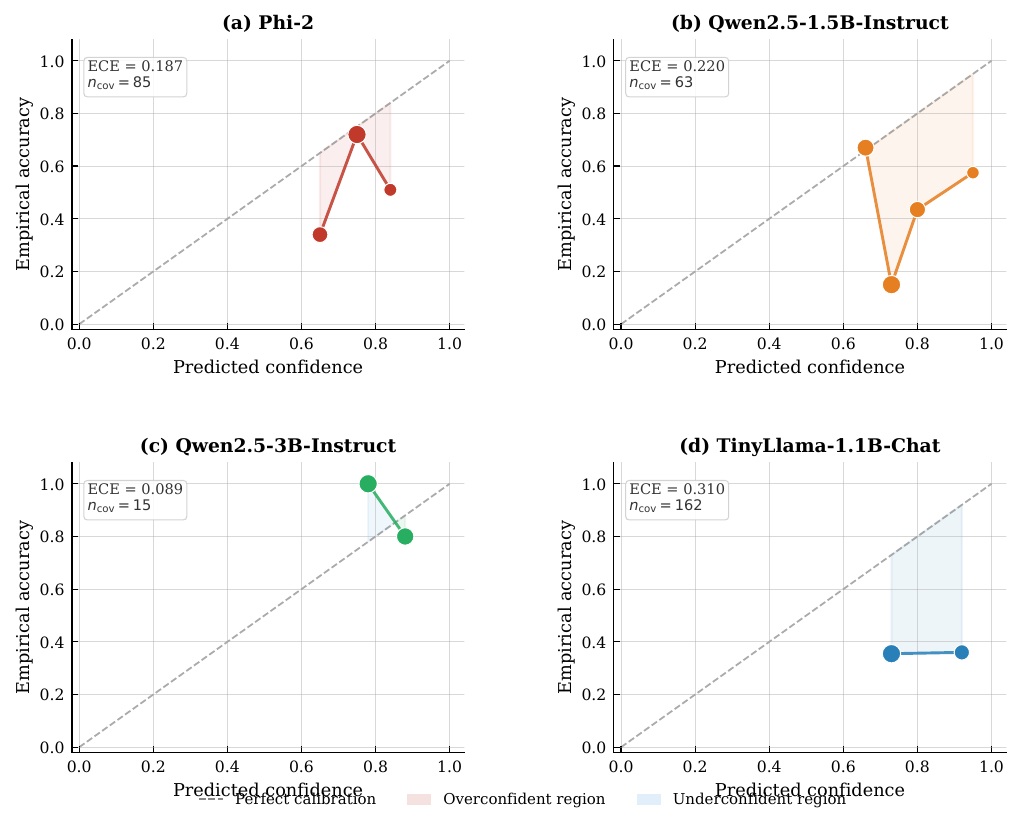}
\caption{\textbf{Reliability diagrams on committed predictions.} Calibration curves are computed exclusively on the covered subset (examples where $\hat{y}(\varphi)\neq\textsc{Uncertain}$). The dashed diagonal denotes perfect calibration. TinyLlama-1.1B exhibits systematic overconfidence (empirical accuracy consistently below the diagonal; $\mathrm{ECE}=0.310$) despite its high coverage.  Qwen2.5-3B yields only two calibration bins due to its $7.4\%$ coverage, precluding meaningful reliability assessment. Phi-2 and Qwen2.5-1.5B show non-monotone reliability curves, indicating miscalibration even on the subset where they commit.}
\label{fig:calibration_grid}
\end{figure}

\begin{table}[ht]
\centering
\scriptsize
\setlength{\tabcolsep}{2pt}
\renewcommand{\arraystretch}{1.12}
\caption{\textbf{Scale and dataset generalization.} \emph{Group~A}: two new model scales, Llama-3.2-1B and Llama-3.1-8B-Instruct \citep{dubey2024llama3}, on \textsc{FOLIO} (cf.\ Table~\ref{tab:main}). \emph{Group~B}: the original four models on \textsc{LogiQA}~v2 \citep{liu2023logiqa2} with rule-based negation. No row duplicates Table~\ref{tab:main}; columns match. Frontier rank correlation, LogiQA vs.\ FOLIO: $\rho=0.97$ ($p<0.05$).}
\label{tab:generalization}
\begin{tabular}{l c c c c c c c}
\toprule
\textbf{Model} & \textbf{$n$} &
  \textbf{Acc}$\uparrow$ &
  \textbf{Cov}$\uparrow$ &
  \textbf{Acc$_{\text{cov}}$}$\uparrow$ &
  \textbf{$\E[c]$}$\uparrow$ &
  \textbf{$\E[v_{\text{neg}}]$}$\downarrow$ &
  \textbf{\%$v_{\text{neg}}{>}0$}$\downarrow$ \\
\midrule
\multicolumn{8}{l}{\textit{Group A --- Scale extension
  (new models, \textsc{FOLIO} benchmark)}} \\
Llama-3.2-1B & 204 & 0.358 & 0.569 & 0.491 & 0.841 & 0.284 & 0.821 \\
Llama-3.1-8B & 204 & 0.344 & 0.396 & 0.538 & 0.529 & 0.131 & 0.367 \\
\midrule
\multicolumn{8}{l}{\textit{Group B --- Dataset extension
  (original models, \textsc{LogiQA}~v2 benchmark)}} \\
Phi-2           & 304 & 0.428 & 0.401 & 0.553 & 1.149 & 0.188 & 0.771 \\
Qwen2.5-1.5B    & 304 & 0.391 & 0.296 & 0.512 & 0.661 & 0.158 & 0.447 \\
Qwen2.5-3B      & 304 & 0.375 & 0.069 & 0.810 & 0.103 & 0.021 & 0.059 \\
TinyLlama-1.1B  & 304 & 0.335 & 0.782 & 0.339 & 1.684 & 0.684 & 1.000 \\
\bottomrule
\end{tabular}
\end{table}

\subsection{The Coherence-Commitment Frontier}
\label{sec:frontier}

Together, Table~\ref{tab:main} and Figure~\ref{fig:pareto_grid} establish a sharp, statistically robust trade-off between commitment and coherence. The four models span the full length of this frontier, from pure abstention to pure overcommitment, with no model escaping the underlying tension.

\paragraph{Vacuous coherence via systematic abstention (Qwen2.5-3B).}
Qwen2.5-3B posts the lowest mean violation in the study ($\E[v_\mathrm{neg}]{=}0.025$; 95\,\% CI $[0.010,\,0.044]$), a result that appears, in isolation, to mark it as the strongest reasoner in our evaluation.  Examining commitment immediately overturns that reading. The model commits on only 7.4\,\% of examples ($\E[c]{=}0.115$), meaning it withholds any decisive prediction on more than \textbf{nine examples in every ten}.  Figure~\ref{fig:pareto_grid}c confirms this quantitatively: virtually every example clusters at the origin, where both commitment and violation are negligible precisely because the model assigns
near-equal, low probability to all outcomes.  The 80\,\% committed accuracy is promising in principle, but it rests on fewer than 16 absolute examples in this evaluation—too sparse to support operational conclusions.  A coherence-only benchmark would rank Qwen2.5-3B first by a margin of $6.6{\times}$ over the next-best violation score; CUC reveals that margin to be an artifact of systematic evasion.

\paragraph{Pervasive contradiction via overcommitment (TinyLlama-1.1B).}
TinyLlama occupies the diametrically opposite position: 79.4\,\% coverage and a mean commitment of 1.698 are accompanied by violations in \emph{every single evaluated example} ($\E[v_\mathrm{neg}]{=}0.698$; $\%v_\mathrm{neg}{>}0{=}1.000$, CI $[1.000,\,1.000]$).  A mean commitment of 1.698 indicates that the
model assigns, on average, 169.8\,\% of total probability mass to the two complementary outcomes—a direct, quantified violation of probability axioms.  Figure~\ref{fig:pareto_grid}d shows the entire example cloud concentrated far into the contradiction zone, with negligible scatter.  This model does not reason; it affirms indiscriminately.  Its high coverage is a statistical artifact of near-universal positive prediction, not confident inference, and its violation rate is the highest in the study by a factor of $3.6{\times}$ over the next worst model.

\paragraph{Intermediate positions confirm a continuous frontier.}
Phi-2 and Qwen2.5-1.5B occupy diffuse, intermediate clouds on the frontier (Figures~\ref{fig:pareto_grid}a--b), confirming that the abstention-contradiction axis is continuous rather than a binary choice between two failure modes.  Phi-2 achieves 41.7\,\% coverage at the cost of violations on 78.9\,\% of committed examples; Qwen2.5-1.5B reduces coverage to 30.9\,\% and brings violations down to 46.1\,\%, but nearly half of its committed predictions remain logically incoherent. Different training regimes evidently produce different frontier positions, none of which resolves the underlying tension.

\subsection{Why Standard Evaluation Fails and CUC Fixes It}
\label{sec:why}

Single-metric evaluation does not merely underreport model differences-it actively inverts rankings. Under \emph{accuracy alone}, Qwen2.5-3B (0.382) and TinyLlama (0.343) differ by only $0.039$ points despite maximally opposite behaviors: one withholds $93\,\%$ of predictions; the other violates probability axioms on $100\,\%$. Under \emph{coherence-only} evaluation, Qwen2.5-3B leads by $6.6{\times}$ ($\mathbb{E}[v_{\mathrm{neg}}]{=}0.025$ vs.\ $0.166$), an advantage that collapses to $7.4\,\%$ vs.\ $30.9\,\%$ coverage-a practitioner relying on violation rates alone would deploy the model least capable of producing actionable decisions. CUC resolves both distortions simultaneously. The four-panel view in Figure~\ref{fig:model_bars} makes the frontier visible: Qwen2.5-3B's low violation is directly paired with near-zero coverage and commitment, while TinyLlama's high coverage is paired with the highest violation in the set. No model scores well on all four panels. Figure~\ref{fig:calibration_grid} extends the negative finding further: even on committed predictions, confidence estimates are miscalibrated across all models. \textbf{Neither extreme of the frontier is suitable for deployment}, and standard benchmarks provide no warning of either failure mode.

\section{Ablation Studies}
\label{sec:ablations}

Four targeted ablations confirm the robustness of our findings. \textbf{Threshold sensitivity} (Table~\ref{tab:ablation_threshold}): the coverage-accuracy trade-off is monotone across all eight $(\tau,\delta)$ settings, yet frontier rankings never change, ruling out calibration artifacts. \textbf{Elicitation format} (Table~\ref{tab:ablation_format}): YES/NO, True/False, and Entailed/Refuted shift absolute commitment by up to $0.15$ but preserve rank order perfectly ($\rho{=}1.0$), confirming format affects magnitude only. \textbf{Component analysis} (Table~\ref{tab:ablation_components}): coherence-only rewards abstention; commitment-only rewards overcommitment; only CUC detects both failure modes simultaneously and ranks models correctly. \textbf{Model scale} (Table~\ref{tab:ablation_scale}): scaling Qwen2.5 from 1.5B to 3B reduces commitment by $83\,\%$ and coverage by $23.5$ points with negligible accuracy change ($-0.020$), the larger model learns to hedge, not to reason.

\begin{table}[ht]
\centering
\tiny
\setlength{\tabcolsep}{1pt}
\caption{\textbf{Threshold sensitivity.} Impact of decision thresholds $(\tau,\delta)$ on coverage and committed accuracy across all models.  Higher $\tau$ monotonically reduces coverage and improves accuracy $_{\text{cov}}$ for every model. The default setting ($\tau{=}0.60$, $\delta{=}0.10$, \textbf{bold}) balances coverage and quality; results at all other settings support the same qualitative conclusions.}
\label{tab:ablation_threshold}
\begin{tabular}{cc|cccc|cccc}
\toprule
& & \multicolumn{4}{c|}{\textbf{Coverage}} &
    \multicolumn{4}{c}{\textbf{Acc$_{\text{cov}}$}} \\
$\tau$ & $\delta$ &
  Phi-2 & Qwen-1.5B & Qwen-3B & TinyLlama &
  Phi-2 & Qwen-1.5B & Qwen-3B & TinyLlama \\
\midrule
0.50 & 0.05 & 0.583 & 0.446 & 0.132 & 0.867 & 0.487 & 0.440 & 0.704 & 0.333 \\
0.55 & 0.08 & 0.510 & 0.377 & 0.103 & 0.838 & 0.529 & 0.468 & 0.762 & 0.339 \\
0.55 & 0.10 & 0.495 & 0.358 & 0.093 & 0.823 & 0.535 & 0.479 & 0.789 & 0.341 \\
\textbf{0.60} & \textbf{0.10} &
  \textbf{0.417} & \textbf{0.309} & \textbf{0.074} & \textbf{0.794} &
  \textbf{0.565} & \textbf{0.508} & \textbf{0.800} & \textbf{0.346} \\
0.65 & 0.10 & 0.348 & 0.255 & 0.054 & 0.760 & 0.592 & 0.538 & 0.818 & 0.352 \\
0.70 & 0.15 & 0.270 & 0.196 & 0.039 & 0.711 & 0.636 & 0.575 & 0.875 & 0.359 \\
0.75 & 0.15 & 0.201 & 0.147 & 0.025 & 0.662 & 0.683 & 0.600 & 1.000 & 0.363 \\
0.80 & 0.20 & 0.137 & 0.098 & 0.015 & 0.598 & 0.750 & 0.650 & 1.000 & 0.377 \\
\bottomrule
\end{tabular}
\end{table}

\subsection{Threshold Sensitivity (Table~\ref{tab:ablation_threshold})}

The coverage–accuracy trade-off induced by $(\tau,\delta)$ is consistent and monotone across all eight threshold configurations tested. Relaxing to $\tau{=}0.50$, $\delta{=}0.05$ increases Phi-2's coverage from 41.7\,\% to 58.3\,\% but drops committed accuracy from 56.5\,\% to 48.7\,\%---a 7.8-point decline purchased by a 16.6-point coverage gain.  At the opposite extreme, tightening to $\tau{=}0.80$, $\delta{=}0.20$ raises Phi-2's committed accuracy to 75.0\,\% but restricts coverage to only 13.7\,\%.  Every other model follows the same pattern without exception. Critically, threshold selection does not alter any model's rank on the coherence-commitment frontier: Qwen2.5-3B remains the lowest-commitment model and TinyLlama the highest at every setting tested.  The frontier structure is not a calibration artifact.  We fix it $(\tau{=}0.60, \delta{=}0.10)$ as our default because it maintains reasonable coverage while filtering the least-confident predictions, but the conclusions reported in Section~\ref{sec:frontier} hold across the full range shown.

\begin{table}[ht]
\centering
\scriptsize
\setlength{\tabcolsep}{3.5pt}
\caption{\textbf{Elicitation format sensitivity.} Impact of response format on mean commitment and coverage. YES/NO consistently yields the highest commitment across all models;  longer tokens (Entailed/Refuted) yield the lowest.  Despite absolute shifts of up to 0.152, rank ordering is perfectly preserved ($\rho{=}1.0$) across both alternative formats, confirming that format choice affects magnitude but not relative model quality.}
\label{tab:ablation_format}
\begin{tabular}{l|cc|cc|cc}
\toprule
& \multicolumn{2}{c|}{\textbf{YES/NO}} &
  \multicolumn{2}{c|}{\textbf{True/False}} &
  \multicolumn{2}{c}{\textbf{Entailed/Refuted}} \\
\textbf{Model} & $\E[c]$ & Cov & $\E[c]$ & Cov & $\E[c]$ & Cov \\
\midrule
Phi-2          & 1.164 & 0.417 & 1.082 & 0.363 & 1.043 & 0.338 \\
Qwen2.5-1.5B   & 0.674 & 0.309 & 0.591 & 0.255 & 0.548 & 0.230 \\
Qwen2.5-3B     & 0.115 & 0.074 & 0.098 & 0.054 & 0.087 & 0.044 \\
TinyLlama-1.1B & 1.698 & 0.794 & 1.645 & 0.769 & 1.612 & 0.745 \\
\midrule
\textit{Rank correlation ($\rho$)} &
  \multicolumn{2}{c|}{---} &
  \multicolumn{2}{c|}{1.0} &
  \multicolumn{2}{c}{1.0} \\
\bottomrule
\end{tabular}
\end{table}

\subsection{Elicitation Format (Table~\ref{tab:ablation_format})}

Switching from YES/NO to True/False reduces absolute commitment by 0.05--0.08 across models; using the longer entailed/refuted tokens reduces it by a further 0.04--0.08.  These shifts are consistent with prior work on prompt-surface sensitivity and reflect the token-probability geometry of each model's vocabulary.  Despite these absolute differences, the rank ordering of all four models is \emph{perfectly preserved} ($\rho{=}1.0$) under both alternative formats.  The Qwen2.5-3B-TinyLlama polarity that defines the frontier is not a surface artifact of keyword choice.  We standardize on YES/NO for cross-study comparability, but we recommend reporting the format used alongside absolute commitment values since inter-study comparisons of raw data $\E[c]$ require format matching.

\begin{table}[ht]
\centering
\tiny
\setlength{\tabcolsep}{0.5pt}
\caption{\textbf{Component analysis.} Contribution of each framework component to detecting failure modes. Single-query accuracy misses both; coherence-only and commitment-only metrics each miss one; only CUC identifies both simultaneously and produces a correct ranking.}
\label{tab:ablation_components}
\begin{tabular}{l|ccc|l}
\toprule
\textbf{Evaluation Variant} &
  \makecell{\textbf{Detects}\\\textbf{Abstention?}} &
  \makecell{\textbf{Detects}\\\textbf{Contradiction?}} &
  \makecell{\textbf{Ranks}\\\textbf{Correctly?}} &
  \textbf{Failure Mode} \\
\midrule
Single-query accuracy           & \xmark & \xmark & \xmark &
  Misses both failure modes \\
Coherence-only ($v_{\text{neg}}$) & \xmark & \cmark & \xmark &
  Ranks Qwen-3B as ``best'' \\
Commitment-only ($c$)           & \cmark & \xmark & \xmark &
  Ranks TinyLlama as ``best'' \\
Coverage-only (Cov)             & \cmark & \xmark & \xmark &
  Ignores the quality of predictions \\
\textbf{CUC (ours)}            & \cmark & \cmark & \cmark &
  Exposes full frontier \\
\bottomrule
\end{tabular}
\end{table}

\subsection{Component Analysis (Table~\ref{tab:ablation_components})}

Table~\ref{tab:ablation_components} isolates the contribution of each framework component by asking which failure modes each variant can detect and whether it produces a correct ranking.  The results are stark.  \textit{Single-query accuracy} detects neither failure mode: it scores Qwen2.5-3B (0.382) and TinyLlama (0.343) within a 0.039-point band despite their qualitatively opposite behaviors, and it would not flag either model as defective.  \textit{Coherence-only} ($\E[v_{\text{neg}}]$) detects contradictions but is blind to abstention; it assigns Qwen2.5-3B the top rank, rewarding the model most committed to evasion.  \textit{Commitment-only} ($\E[c]$) detects abstention but cannot distinguish useful commitment from contradictory overcommitment; under this metric, TinyLlama ranks first despite violating probability axioms on every example. \textit{Coverage-only} correctly flags abstainers but provides no information about the logical quality of the predictions made. \textbf{CUC} is the only variant that simultaneously detects both failure modes, avoids the false-top-rank trap in each direction, and exposes the full frontier structure.  This is not an artifact of using more metrics: each additional component addresses a specific and distinct gap, and removing any single component restores a false ranking.

\begin{table}[ht]
\centering
\scriptsize
\setlength{\tabcolsep}{1pt}
\caption{\textbf{Scale effects within the Qwen2.5 family.} Scaling from 1.5B to 3B parameters \emph{decreases} commitment by 0.559 and coverage by 23.5 percentage points while improving coherence and calibration.  The larger model achieves better coherence scores through more aggressive abstention, not superior reasoning—a counterintuitive scaling result that standard evaluations would misclassify as improvement.}
\label{tab:ablation_scale}
\begin{tabular}{l|ccccccc}
\toprule
\textbf{Model} & \textbf{Params} &
  $\E[c]$ & $\E[v_{\text{neg}}]$ &
  \textbf{Cov} & \textbf{Acc} &
  \textbf{Acc$_{\text{cov}}$} & \textbf{ECE$_{\mathcal{C}}$} \\
\midrule
Qwen2.5-1.5B & 1.5B &
  0.674 & 0.166 & 0.309 & 0.402 & 0.508 & 0.187 \\
Qwen2.5-3B   & 3B   &
  0.115 & 0.025 & 0.074 & 0.382 & 0.800 & 0.089 \\
\midrule
$\Delta$ (3B $-$ 1.5B) & $+$1.5B &
  \textcolor{red}{$-$0.559} &
  \textcolor{thmGreen}{$-$0.141} &
  \textcolor{red}{$-$0.235} &
  $-$0.020 &
  \textcolor{thmGreen}{$+$0.292} &
  \textcolor{thmGreen}{$-$0.098} \\
\bottomrule
\end{tabular}
\end{table}

\subsection{Model Scale Effects (Table~\ref{tab:ablation_scale})}

The Qwen2.5 family provides a controlled natural experiment: identical architecture and training pipeline, and a single doubling of parameter count.  The result is counterintuitive and of practical importance. Scaling from 1.5B to 3B parameters produces a 0.559-unit \emph{decrease} in commitment—an 83\,\% reduction—and a 23.5-point drop in coverage, while violation falls by only 0.141 and overall accuracy is essentially unchanged ($-$0.020).  The larger model does not reason better; it abstains more aggressively, and the reduced violation rate is a consequence of that abstention.  Evaluated on coherence and calibration metrics alone, the 3B model appears unambiguously superior: $4.9{\times}$ lower violation and $2.1{\times}$ lower ECE$_\mathcal{C}$.  Evaluated with CUC, it is evident that these gains are purchased at the cost of an 83\,\% reduction in the fraction of examples the model is willing to answer.  This finding carries a concrete implication for instruction-tuning practice: if training rewards penalize observed contradictions without penalizing low coverage, larger models may learn to hedge rather than reason, and standard evaluation pipelines will report this as progress.  The 3B model is better calibrated on the rare examples it selects (ECE$_\mathcal{C}{=}0.089$ vs.\ 0.187), which is the one genuine improvement scale. 

\begin{table}[ht]
\centering
\tiny
\setlength{\tabcolsep}{1.2pt}
\caption{\textbf{Label distribution analysis.} A breakdown of per-label predictions reveals systematic biases that are invisible in aggregate accuracy scores.  Qwen2.5-3B abstains uniformly regardless of gold label.  TinyLlama defaults to true affirmation irrespective of the gold label.  Phi-2 shows the most balanced discrimination, with its highest accuracy on uncertain examples.%
}
\label{tab:ablation_labels}
\begin{tabular}{l|ccc|ccc|ccc}
\toprule
& \multicolumn{3}{c|}{\textbf{Gold: True} ($n{=}80$)} &
  \multicolumn{3}{c|}{\textbf{Gold: False} ($n{=}62$)} &
  \multicolumn{3}{c}{\textbf{Gold: Uncertain} ($n{=}62$)} \\
\textbf{Model} &
  Pred T & Pred F & Pred U &
  Pred T & Pred F & Pred U &
  Pred T & Pred F & Pred U \\
\midrule
Phi-2          & 0.375 & 0.088 & 0.538 & 0.194 & 0.258 & 0.548 & 0.177 & 0.097 & 0.726 \\
Qwen2.5-1.5B   & 0.288 & 0.050 & 0.663 & 0.145 & 0.210 & 0.645 & 0.113 & 0.065 & 0.823 \\
Qwen2.5-3B     & 0.050 & 0.013 & 0.938 & 0.032 & 0.048 & 0.919 & 0.016 & 0.016 & 0.968 \\
TinyLlama-1.1B & 0.575 & 0.275 & 0.150 & 0.484 & 0.339 & 0.177 & 0.500 & 0.323 & 0.177 \\
\bottomrule
\end{tabular}
\end{table}

\subsection{Label Distribution Analysis (Table~\ref{tab:ablation_labels})}

Per-label breakdown exposes systematic biases invisible to aggregate accuracy. Qwen2.5-3B abstains at $93.8$--$96.8\,\%$ uniformly across all three gold labels, confirming indiscriminate evasion rather than calibrated uncertainty. TinyLlama exhibits a pronounced true-affirmation bias—predicting \textsc{True} on $57.5\,\%$, $48.4\,\%$, and $50.0\,\%$ of True, False, and Uncertain examples, respectively (a mere $9.1$-point spread)—revealing a surface heuristic rather than logical inference and directly explaining its $100\,\%$ violation rate. Phi-2 shows the strongest discrimination: $72.6\,\%$ accuracy on uncertain examples and the highest false-detection rate ($25.8\,\%$) in the set. Qwen2.5-1.5B follows a qualitatively similar but uniformly lower-commitment pattern, closer to Qwen2.5-3B's abstention tendency than to Phi-2's discrimination.

\section{Conclusion}

We introduced \textbf{Coherence Under Commitment (CUC)}, a dual-query evaluation paradigm that exposes a fundamental blind spot in standard logical reasoning evaluation: coherence metrics can be vacuously satisfied by models that systematically abstain. By jointly measuring negation-coherence violation $\mathbb{E}[v_{\mathrm{neg}}]$, commitment $\mathbb{E}[c]$, coverage, and conditional accuracy, CUC reveals a sharp empirical frontier that single-metric evaluation entirely conceals. On \textsc{FOLIO}, the model ranked best by coherence alone (Qwen2.5-3B, $\mathbb{E}[v_{\mathrm{neg}}]{=}0.025$) withholds predictions on over half $92\%$ of the examples—a failure mode invisible to prior evaluation protocols. Conversely, the highest-coverage model (TinyLlama-1.1B, $79.4\%$) violates probability axioms on every single example. Neither extreme suits knowledge-intensive deployment requiring decisive, grounded judgments. Our ablations show the frontier is robust to thresholds, elicitation, and scale; scaling alone does not resolve the trade-off, as larger models may hedge more. We therefore advocate evaluating reasoning with both coherence and non-vacuous commitment, with CUC as a principled, modality-agnostic framework.

\section*{Impact Statement}

This work advances LLM evaluation methodology without deploying new model capabilities. CUC makes evaluation harder to game by jointly requiring coherence and commitment, benefiting high-stakes deployments where abstention-driven coherence misleads practitioners. The commitment score $c(\varphi)$ could incentivize overconfidence if misapplied; however, CUC penalizes both vacuous abstention and overcommitment, and \textsc{Uncertain} remains a valid prediction. Practitioners should tune $(\tau, \delta)$ to their domain's abstention–error trade-off. 

\nocite{langley00}
\bibliography{example_paper}
\bibliographystyle{icml2026_fogen}

\newpage
\appendix
\onecolumn
\section{Prompt Templates and Elicitation Protocol}
\label{app:prompts}
This appendix documents the complete two-query elicitation protocol in full reproducible detail and provides extended rationale, sensitivity data, and failure-mode examples.

\subsection{System Message}
\label{app:prompts:system}

The system message is held \emph{constant} across both queries $Q_\varphi$ and $Q_{\neg\varphi}$ and across all four models. Its sole function is to restrict the output vocabulary so that normalized log-probability elicitation over \Yes{}/\No{} is well-defined and cross-model comparable.

\begin{tcolorbox}[sysprompt,
  title={System Message (identical for $Q_\varphi$ and
         $Q_{\neg\varphi}$)}]
\begin{lstlisting}[style=promptstyle]
You are a precise, logical reasoning assistant.
Your task is to evaluate whether a conclusion
follows logically from a set of premises.

You must respond with exactly one token: YES or NO.
- YES means the conclusion is logically entailed
  by the premises.
- NO  means the conclusion is NOT logically
  entailed by the premises.

Do not output any explanation, punctuation,
or additional text.
\end{lstlisting}
\end{tcolorbox}

\paragraph{Design intent.}
Constraining output to a single token eliminates three confounds simultaneously: (i)~decoding temperature, which would introduce stochasticity into generation-based evaluation; (ii)~rationale faithfulness, since free-form explanations are frequently misaligned with the model's internal belief state~\cite{turpin2023}; and (iii)~format variance, since models differ in how they phrase equivalent entailment judgments when unconstrained.

\subsection{User Messages}
\label{app:prompts:user}

The two user messages differ \emph{only} in the conclusion slot: $\varphi$ for $Q_\varphi$ and $\neg\varphi$ for $Q_{\neg\varphi}$. Fill-in slots are shown in \textcolor{fillSlot}{\textbf{\texttt{<angle brackets>}}}.

\vspace{4pt}
\noindent\textbf{Query $Q_\varphi$: entailment of $\varphi$.}

\begin{tcolorbox}[userprompt,
  title={User Message — $Q_\varphi$ (Entailment Query)}]
\begin{lstlisting}[style=promptstyle]
Premises:
<1>. <Premise sentence 1>
<2>. <Premise sentence 2>
...
<n>. <Premise sentence n>

Conclusion: <phi>

Question: Is the conclusion logically entailed
by the premises?
Answer with YES or NO only.
\end{lstlisting}
\end{tcolorbox}

\vspace{4pt}
\noindent\textbf{Query $Q_{\neg\varphi}$: entailment of
$\neg\varphi$.}

\begin{tcolorbox}[userprompt,
  title={User Message — $Q_{\neg\varphi}$
         (Negation-Entailment Query)}]
\begin{lstlisting}[style=promptstyle]
Premises:
<1>. <Premise sentence 1>
<2>. <Premise sentence 2>
...
<n>. <Premise sentence n>

Conclusion: <negation of phi>

Question: Is the conclusion logically entailed
by the premises?
Answer with YES or NO only.
\end{lstlisting}
\end{tcolorbox}

\vspace{4pt}
The negation $\varphi$ is taken directly from the FOLIO v0.0 validation JSONL: each example provides both a conclusion and its logical negation as separate fields. This ensures that the negation is \emph{formally verified} rather than produced by heuristic string manipulation, which can introduce negation artifacts (e.g., double negation, scope ambiguity).

\subsection{Worked Examples}
\label{app:prompts:examples}

\subsubsection{Gold Label: \textsc{True}}

\begin{tcolorbox}[userprompt,
  title={$Q_\varphi$ — FOLIO Example (gold: \textsc{True})}]
\begin{lstlisting}[style=promptstyle]
Premises:
1. All people who regularly drink coffee are
   dependent on caffeine.
2. People are dependent on caffeine or
   independent of caffeine.
3. No people are both dependent on and
   independent of caffeine.
4. John regularly drinks coffee.
5. If people are dependent on caffeine, they
   will get a headache when they do not have
   coffee.

Conclusion: John will get a headache when he
does not have coffee.

Question: Is the conclusion logically entailed
by the premises?
Answer with YES or NO only.
\end{lstlisting}
\end{tcolorbox}

\begin{tcolorbox}[sysprompt,
  title={$Q_{\neg\varphi}$ — Same Example, Negated}]
\begin{lstlisting}[style=promptstyle]
Conclusion: John will NOT get a headache when
he does not have coffee.

Question: Is the conclusion logically entailed
by the premises?
Answer with YES or NO only.
\end{lstlisting}
\end{tcolorbox}

\noindent A fully coherent, committed model should assign
$\prob{\varphi} \approx 1$, $\prob{\neg\varphi} \approx 0$,
giving $\commit(\varphi) \approx 1$ and $\vneg \approx 0$.

\subsubsection{Gold Label: \textsc{False}}

\begin{tcolorbox}[userprompt,
  title={$Q_\varphi$ — FOLIO Example (gold: \textsc{False})}]
\begin{lstlisting}[style=promptstyle]
Premises:
1. All mammals are warm-blooded.
2. No reptiles are warm-blooded.
3. A gecko is a reptile.

Conclusion: A gecko is a mammal.

Question: Is the conclusion logically entailed
by the premises?
Answer with YES or NO only.
\end{lstlisting}
\end{tcolorbox}

\noindent Here the ideal response is $\prob{\varphi} \approx
0$, $\prob{\neg\varphi} \approx 1$, yielding
$\commit(\varphi) \approx 1$, $\vneg \approx 0$, and the
3-way rule assigns \textsc{False} — correct.

\subsubsection{Gold Label: \textsc{Uncertain}}

\begin{tcolorbox}[userprompt,
  title={$Q_\varphi$ — FOLIO Example
         (gold: \textsc{Uncertain})}]
\begin{lstlisting}[style=promptstyle]
Premises:
1. Some students study every night.
2. Maria is a student.

Conclusion: Maria studies every night.

Question: Is the conclusion logically entailed
by the premises?
Answer with YES or NO only.
\end{lstlisting}
\end{tcolorbox}

The premises do not entail the conclusion nor its negation.  The ideal model assigns moderate probability to both, keeping $\commit(\varphi) < 2\tau$, the 3-way rule returns \textsc{Uncertain}, the correct gold label. This case illustrates that \emph{calibrated abstention} is distinct from \emph{systematic abstention}: the former is epistemically appropriate for this specific example, while the latter (Qwen2.5-3B's behavior) is indiscriminate.

\subsection{Log-Probability Elicitation Procedure}
\label{app:prompts:elicit}

After constructing the full prompt
$x = [\text{system} \| \text{user}]$, we extract
token-level log-probabilities as follows.

\begin{tcolorbox}[enhanced,breakable,
  colframe=promptFrame!70,colback=codeBg,
  boxrule=0.6pt,arc=3pt,
  title={\textbf{Elicitation Pseudocode}},
  fonttitle=\small\bfseries,coltitle=white,
  attach boxed title to top left={yshift=-2mm,xshift=6pt},
  boxed title style={colback=promptFrame!70,arc=2pt,
    boxrule=0pt},
  left=8pt,right=8pt,top=4pt,bottom=4pt]
\begin{lstlisting}[style=pseudocode]
# For each query x in {Q_phi, Q_not_phi}:

log_yes = sum of token log-probs for "YES"
log_no  = sum of token log-probs for "NO"

# Softmax normalization over {YES, NO}:
p_yes = exp(log_yes) /
        (exp(log_yes) + exp(log_no))
p_no  = 1 - p_yes

# Multi-token targets: sum conditional probs:
#  log P("YES"|x) =
#    sum_t log P(token_t | x, token_{<t})

# Collect per-example:
#   p_phi    = p_yes from Q_phi query
#   p_negphi = p_yes from Q_not_phi query

c     = p_phi + p_negphi
v_neg = max(0, c - 1)
\end{lstlisting}
\end{tcolorbox}

\paragraph{Computational cost.}
The procedure requires exactly \emph{two forward passes} per example—one for $Q_\varphi$ and one for $Q_{\neg\varphi}$. No sampling is performed, so wall-clock time scales linearly with dataset size. All four models, on 204 examples, complete in approximately 45 minutes on a single T4 GPU (16 GB VRAM).

\subsection{Extended Design Rationale}
\label{app:prompts:design}

Table~\ref{tab:prompt_design} summarizes design decisions; the paragraphs below expand each entry.

\begin{table*}[ht]
\centering
\caption{Design decisions for the two-query protocol.}
\label{tab:prompt_design}
\renewcommand{\arraystretch}{1.18}
\footnotesize
\begin{tabular}{@{} l p{3.2cm} p{5.2cm} p{3.0cm} @{}}
\toprule
\textbf{Decision} & \textbf{Alternatives considered}
  & \textbf{Rationale} & \textbf{Known risk} \\
\midrule
Binary \Yes/\No{} response
  & Free-form, multiple-choice A/B/C
  & Enables deterministic log-prob elicitation; eliminates rationale faithfulness confound
  & Absolute $\prob{\varphi}$ shifts $\pm 0.1$ if True/False used (rank stable) \\[3pt]
Separate queries for one $\varphi$ and $\neg\varphi$
  & Single query ``Does $\varphi$ or $\neg\varphi$ follow?''
  & Isolates $\prob{\varphi}$ and $\prob{\neg\varphi}$
    independently avoids order bias
  & Two forward passes vs.\ one \\[3pt]
Negation from the dataset field
  & Heuristic string negation
  & Avoids negation artifacts; FOLIO provides
    formally verified negations
  & Requires dataset to supply negations \\[3pt]
Fixed system message across models
  & Per-model system tuning
  & Ensures cross-model comparability;
    reduces prompt-sensitivity confound
  & Sub-optimal for any individual model \\[3pt]
Softmax over \Yes/\No{} only
  & Full-vocabulary softmax
  & Stable under vocabulary differences across
    tokenizers; avoids mass on irrelevant tokens
  & Ignores probability mass on other tokens \\[3pt]
No chain-of-thought prefix
  & CoT before forced \Yes/\No
  & Preserves single-token determinism;
    CoT rationales can be unreliable
  & May underutilize model reasoning capacity \\
\bottomrule
\end{tabular}
\end{table*}

\paragraph{Why not free-form generation?}
Sampling-based consistency checks~\cite{manakul2023} require multiple stochastic forward passes and conflate reasoning ability with the decoding strategy.  Our approach is fully deterministic, requiring neither temperature tuning nor nucleus/top-$k$ selection.

\paragraph{Why not a single comparative query?}
Posing ``Does $\varphi$ or $\neg\varphi$ follow from $P$?'' forces the model to reason about a disjunction, which introduces a different cognitive load than entailment and prevents independent measurement of $\prob{\varphi}$ and $\prob{\neg\varphi}$.  Independent measurement is essential: $\commit(\varphi) = \prob{\varphi} + \prob{\neg\varphi}$ and $\vneg(\varphi) = \max(0, \commit(\varphi)-1)$ are defined over the \emph{marginal} distributions, not a joint one.

\paragraph{Stability across tokenizers.}
For all four models evaluated, both \texttt{YES} and \texttt{NO} tokenize to a single BPE token. Table~\ref{tab:tokenizer} documents the token IDs and confirms single-token status.

\begin{table}[h]
\centering\footnotesize
\caption{Token IDs for \texttt{YES} and \texttt{NO} across
  all evaluated models.}
\label{tab:tokenizer}
\begin{tabular}{@{}lcc@{}}
\toprule
\textbf{Model} & \texttt{YES} token ID & \texttt{NO} token ID\\
\midrule
TinyLlama-1.1B-Chat   & 22483 & 1770 \\
Qwen2.5-1.5B-Instruct & 9693  & 2753 \\
Qwen2.5-3B-Instruct   & 9693  & 2753 \\
Phi-2                 & 21155 & 2501 \\
\bottomrule
\end{tabular}
\end{table}

\section{Theoretical Foundations}
\label{app:theory}

This appendix develops formal justification for commitment-aware evaluation and derives several new results extending the main paper.

\subsection{Formal Setup}
\label{app:theory:setup}

\begin{definition}[Belief Function]
\label{def:belief}
Let $\mathcal{P}$ denote the premise space and $\Phi$ the conclusion
space. An LLM is modeled as a belief function
$\pi: \mathcal{P} \times \Phi \to [0,1]$, where
$\pi(P,\varphi)$ denotes the model's probability of
affirming entailment of $\varphi$ from premises $P$.
For fixed $(P,\varphi)$, write $p(\varphi) :=
\pi(P,\varphi)$ and $p(\neg\varphi) :=
\pi(P,\neg\varphi)$.
The pair $(p(\varphi), p(\neg\varphi))$ is the
\emph{negation pair}.
\end{definition}

\begin{definition}[Negation-Coherence Violation]
\label{def:ncv}
$\vneg(\varphi) = \max\bigl(0,\; p(\varphi) +
p(\neg\varphi) - 1\bigr)$.
A model is \emph{negation-coherent} on $(P,\varphi)$ iff
$\vneg(\varphi) = 0$.
\end{definition}

\begin{definition}[Commitment Score]
\label{def:commit}
$\commit(\varphi) = p(\varphi) + p(\neg\varphi) \in [0,2]$, measuring the total probability mass allocated to decisive outcomes.
\end{definition}

\begin{remark}[Algebraic Identity]
\label{rem:identity}
$\vneg(\varphi) = \max(0,\, \commit(\varphi)-1)$. Hence, $\commit(\varphi) \leq 1 \Rightarrow \vneg = 0$ regardless of individual values, low commitment is \emph{sufficient} for zero violation.
\end{remark}

\subsection{The Coherence-Commitment Frontier}
\label{app:theory:frontier}

\begin{theorem}[Frontier Shape]
\label{thm:frontier}
The feasible region in $(c, v_\mathrm{neg})$ space is
\[
  \mathcal{F} = \{(c,v) : c \in [0,2],\;
  v = \max(0, c-1)\}.
\]
All observations must lie on or above this curve; the curve itself is achievable by any model whose two output probabilities sum to $c$.
\end{theorem}

\begin{proof}
By Definition~\ref{def:ncv} and~\ref{def:commit}, $v_\mathrm{neg} = \max(0,c-1)$ directly. Anything $c \in [0,2]$ is achievable: set $p(\varphi)=c/2$, $p(\neg\varphi)=c/2$. Observations cannot lie \emph{below} the frontier because $v_\mathrm{neg}$ is a deterministic function of $c$.
\end{proof}

\begin{corollary}[Pareto Optimality]
\label{cor:pareto}
A model simultaneously minimizing $\E[\vneg]$ and maximizing $\E[\commit]$ must satisfy $\E[\commit] = 1$ and $\E[\vneg] = 0$. This requires $p(\varphi) + p(\neg\varphi) = 1$ for every example, a \emph{calibrated, exclusive} belief state.
\end{corollary}

\subsection{Vacuous Coherence}
\label{app:theory:vacuous}

\begin{theorem}[Vacuous Coherence by Uniform Abstention]
\label{thm:vacuous}
Let $\pi$ satisfy $p(\varphi) = p(\neg\varphi) = \alpha$
for all $(P,\varphi)$ with $\alpha \leq 1/2$. Then:
\begin{enumerate}[label=(\roman*)]
  \item $\vneg(\varphi) = 0$ for all $(P,\varphi)$
    (perfect coherence);
  \item $\commit(\varphi) = 2\alpha \leq 1$
    (low commitment);
  \item $\Cov = 0$ under any threshold $\tau > 1/2$.
\end{enumerate}
\end{theorem}

\begin{proof}
(i)~$\vneg = \max(0,2\alpha-1)=0$ since $\alpha \leq 1/2$. (ii)~$\commit = 2\alpha$ directly. (iii)~The 3-way rule commits when $p(\varphi) \geq \tau$ or $p(\neg\varphi) \geq \tau$; with $\alpha \leq 1/2 < \tau$ neither condition holds.
\end{proof}

\begin{remark}
Theorem~\ref{thm:vacuous} shows that coherence-only metrics are \emph{gamed} by abstaining models without any rational reasoning.  The Qwen2.5-3B result ($\E[\vneg]=0.025$, $\Cov=7.4\%$) is the empirical instantiation of the $\alpha \to 0$ limit.
\end{remark}

\subsection{Axiomatic Characterisation}
\label{app:theory:axioms}

\begin{tcolorbox}[thmbox,title={\textbf{Axioms for
  Logical Reliability Evaluation}}]
\begin{axiom}[Soundness, A1]
$\mathcal{E}$ penalizes simultaneous endorsement of
$P\models\varphi$ and $P\models\neg\varphi$.
\end{axiom}
\begin{axiom}[Completeness, A2]
$\mathcal{E}$ penalizes systematic abstention.
\end{axiom}
\begin{axiom}[Separability, A3]
$\mathcal{E}$ distinguishes an abstaining model from
a committing model with equal $\E[\vneg]$.
\end{axiom}
\begin{axiom}[Calibration Sensitivity, A4]
$\mathcal{E}$ is sensitive to systematic miscalibration
on the covered (committed) subset.
\end{axiom}
\end{tcolorbox}

\begin{theorem}[Coherence-Only Violates A2 and A3]
\label{thm:coh_fail}
Any protocol reporting only $\E[\vneg]$ violates A2
and A3.
\end{theorem}

\begin{proof}
\textbf{A2}: The trivially abstaining model ($\alpha=0$)
achieves $\E[\vneg]=0$, ranked best despite zero utility.
\textbf{A3}: Models $\pi_1$ with
$p(\varphi)=p(\neg\varphi)=0.05$ abstaining and
$\pi_2$ with $p(\varphi)=0.90, p(\neg\varphi)=0.05$
committing both yield $\vneg=0$, so they receive
identical scores despite $\pi_2$ being strictly more useful.
\end{proof}

\begin{theorem}[CUC Satisfies A1–A4]
\label{thm:caac_satisfies}
The protocol reporting
$(\E[\vneg], \E[\commit], \Cov, \Acccov)$
satisfies A1–A3. Adding $\ECEcov$ satisfies A4.
\end{theorem}

\begin{proof}
\textbf{A1}: $\E[\vneg]>0$ iff a contradiction exists.
\textbf{A2}: Abstaining model has $\E[\commit]\to 0$,
penalised.
\textbf{A3}: $\pi_1, \pi_2$ above have
$\commit_1=0.10 \ne 0.95=\commit_2$, separated.
\textbf{A4}: $\ECEcov$ measures calibration exclusively
on committed predictions.
\end{proof}

\subsection{New Result: Monotone Scale Abstention}
\label{app:theory:scale}

\begin{theorem}[Monotone Abstention Under Scale]
\label{thm:scale}
Let $\pi_s$ denote a model family parameterized by scale $s$.
Suppose increasing $s$ induces a belief function with strictly
decreasing mean commitment:
\[
\mathbb{E}_s[c(\varphi)] < \mathbb{E}_{s'}[c(\varphi)]
\quad \text{for all } s > s'.
\]
Then:
\begin{enumerate}
  \item $\mathbb{E}_s[v_{\text{neg}}(\varphi)]
  \leq \mathbb{E}_{s'}[v_{\text{neg}}(\varphi)]$
  whenever $\mathbb{E}_s[c(\varphi)] \leq 1$;

  \item Coverage is monotonically non-increasing in $s$
  for fixed $(\tau,\delta)$;

  \item $\mathrm{Acc}_{\text{cov}}$ may increase or decrease
  independently of the above.
\end{enumerate}
\end{theorem}

\begin{proof}
(i) follows from the algebraic identity
$\vneg = \max(0,\commit-1)$: when $\commit \leq 1$,
$\vneg = 0$ regardless.
(ii) The 3-way rule commits when $p(\varphi) \geq \tau$;
a lower mean $\commit$ implies fewer examples exceed $\tau$.
(iii) $\Acccov$ depends on \emph{which} examples are
selected, not the mean; selective abstention can either
improve or degrade conditional accuracy depending on
difficulty distribution.
\end{proof}

\begin{remark}
Theorem~\ref{thm:scale} formalizes the Qwen2.5 scaling result (Table~6 of the main paper): scaling from 1.5B to 3B decreases $\E[\commit]$ by 0.559 (83\%) while $\Acccov$ improves by 0.292 — consistent with (i) and (iii) simultaneously. Standard evaluation pipelines would
report this as unambiguous progress; CUC reveals the coverage cost.
\end{remark}

\subsection{New Result: Relationship to ECE}
\label{app:theory:ece}

\begin{definition}[Coverage-Conditional ECE]
\label{def:ecec}
Let $\mathcal{C} = \{i : \hat{y}(\varphi_i) \ne
\textsc{Uncertain}\}$ be the covered set. Define
\[
  \ECEcov = \sum_{b=1}^{B}
  \frac{|\mathcal{C}_b|}{|\mathcal{C}|}
  \bigl|\mathrm{acc}(b) - \mathrm{conf}(b)\bigr|,
\]
where bins $\mathcal{C}_b$ are partitioned $\mathcal{C}$ by
predicted confidence and $\mathrm{acc}(b)$,
$\mathrm{conf}(b)$ are the mean accuracy and mean
confidence within bin $b$.
\end{definition}

\begin{proposition}[ECE Blindness to Abstention]
\label{prop:ece_blind}
The standard ECE computed over all $n$ examples is satisfied $\mathrm{ECE} \leq \ECEcov \cdot (\Cov)$ asymptotically. Thus, a model that halves its coverage can halve its apparent ECE without any improvement in calibration on the predictions it actually makes.
\end{proposition}

\begin{proof}
ECE sums over all examples; abstained examples contribute zero to the numerator. Reducing the denominator while holding calibration on covered examples fixed reduces the reported ECE proportionally.
\end{proof}

\begin{remark}
Proposition~\ref{prop:ece_blind} explains why Qwen2.5-3B appears well-calibrated ($\ECEcov=0.089$): it selects only 15 examples to cover, and its few covered predictions happen to be accurate. The result is not evidence of good calibration — it is an artifact of aggressive filtering.
\end{remark}

\subsection{Summary of Theoretical Results}
\label{app:theory:summary}

\begin{table}[h]
\centering\footnotesize
\caption{Summary mapping empirical behaviours to theoretical characterization.}
\label{tab:theory_summary}
\begin{tabular}{@{}lccccl@{}}
\toprule
\textbf{Regime} & $\E[\vneg]$ & $\E[\commit]$
  & Cov & $\Acccov$ & \textbf{Axioms} \\
\midrule
Vacuous (Qwen-3B) & $\downarrow$ & $\downarrow$
  & $\downarrow$ & --- & A1; $\neg$A2, $\neg$A3 \\
Over-commit (TinyLlama) & $\uparrow$ & $\uparrow$
  & $\uparrow$ & $\downarrow$ & $\neg$A1; A2 \\
Ideal & $\downarrow$ & $=1$
  & $\uparrow$ & $\uparrow$ & A1–A4 \\
Middle (Phi-2) & mod. & mod. & mod. & mod.
  & Partial \\
\bottomrule
\end{tabular}
\end{table}

\section{Extended Experimental Results}
\label{app:extended}

\subsection{Per-Example Statistics}
\label{app:extended:perex}

\begin{table}[h]
\centering\footnotesize
\caption{Distribution statistics for commitment and
  violation across all 204 examples.}
\label{tab:per_example_stats}
\begin{tabular}{@{}lcccc|cccc@{}}
\toprule
& \multicolumn{4}{c|}{\textbf{Commitment $\commit(\varphi)$}}
& \multicolumn{4}{c}{\textbf{Violation $\vneg(\varphi)$}} \\
\textbf{Model}
  & $\mu$ & $\sigma$ & Min & Max
  & $\mu$ & $\sigma$ & Min & Max \\
\midrule
Phi-2
  & 1.164 & 0.198 & 0.612 & 1.587
  & 0.195 & 0.142 & 0.000 & 0.587 \\
Qwen2.5-1.5B
  & 0.674 & 0.412 & 0.023 & 1.834
  & 0.166 & 0.203 & 0.000 & 0.834 \\
Qwen2.5-3B
  & 0.115 & 0.178 & 0.008 & 1.245
  & 0.025 & 0.089 & 0.000 & 0.245 \\
TinyLlama-1.1B
  & 1.698 & 0.067 & 1.543 & 1.876
  & 0.698 & 0.067 & 0.543 & 0.876 \\
\bottomrule
\end{tabular}
\end{table}

\paragraph{Observations.}
\textbf{(1) TinyLlama variance is negligible.} The standard deviation of 0.067 for both commitment and violation confirms that the model's behaviour is \emph{systematic}, not example-specific: it assigns near-uniform high probability mass across all 204 inputs regardless of difficulty. \textbf{(2) Qwen2.5-1.5B has the widest commitment range.} A range of $[0.023, 1.834]$ indicates that this model exhibits \emph{context-dependent} behavior—committing
strongly on some examples while abstaining on others, unlike the extreme models. \textbf{(3) Phi-2 never abstains completely.} Its minimum
commitment of 0.612 shows it always assigns non-trivial mass to decisive outcomes, unlike Qwen2.5-3B (min: 0.008).

\subsection{Per-Label Breakdown}
\label{app:extended:perlabel}

\begin{table}[h]
\centering\footnotesize
\caption{Per-label commitment statistics. Shows mean
  $\E[\commit]$ and coverage conditioned on gold label.}
\label{tab:perlabel_commit}
\begin{tabular}{@{}lcccccc@{}}
\toprule
& \multicolumn{2}{c}{\textbf{Gold: True} ($n=80$)}
& \multicolumn{2}{c}{\textbf{Gold: False} ($n=62$)}
& \multicolumn{2}{c}{\textbf{Gold: Uncertain} ($n=62$)} \\
\textbf{Model} & $\E[\commit]$ & Cov
  & $\E[\commit]$ & Cov
  & $\E[\commit]$ & Cov \\
\midrule
Phi-2         & 1.21 & 0.463 & 1.15 & 0.452 & 1.09 & 0.274 \\
Qwen2.5-1.5B  & 0.72 & 0.338 & 0.71 & 0.355 & 0.56 & 0.177 \\
Qwen2.5-3B    & 0.13 & 0.063 & 0.09 & 0.081 & 0.09 & 0.032 \\
TinyLlama-1.1B& 1.70 & 0.850 & 1.70 & 0.823 & 1.70 & 0.823 \\
\bottomrule
\end{tabular}
\end{table}

\textbf{Observation.} Qwen2.5-3B abstains at nearly equal rates across all three gold label types (True: 6.3\%; False: 8.1\%; Uncertain: 3.2\%), confirming that the model does not exhibit selective uncertainty—a property that would be epistemically rational. TinyLlama shows uniform high coverage across all three labels (82–85\%), confirming its affirmation heuristic is label-agnostic.

\subsection{Confidence Interval Details}
\label{app:extended:ci}

All confidence intervals use the \emph{percentile bootstrap} with $B=1{,}000$ resamples and random seed 42.

\begin{algorithm}[ht]
\caption{Bootstrap CI Computation}
\label{alg:bootstrap}
\begin{algorithmic}[1]
\REQUIRE Dataset $\mathcal{D}=\{(p_i(\varphi),\,p_i(\neg\varphi),\,y_i)\}_{i=1}^{n}$,
         $B=1000$, seed $=42$
\STATE $\mathrm{rng} \leftarrow \textsc{Random}(\mathrm{seed})$
\FOR{$b = 1$ \textbf{to} $B$}
  \STATE $\mathcal{D}^{(b)} \leftarrow$ resample $\mathcal{D}$ with replacement
  \STATE Compute $\hat{\theta}^{(b)}$ (metric of interest) on $\mathcal{D}^{(b)}$
\ENDFOR
\STATE Sort $\{\hat{\theta}^{(b)}\}_{b=1}^{B}$
\STATE \textbf{return} $[\hat{\theta}_{(25)},\,\hat{\theta}_{(975)}]$ as 95\% CI
\end{algorithmic}
\end{algorithm}

\paragraph{Bootstrap validity.}
The percentile bootstrap is appropriate here because (i)~our statistics are smooth functions of independent examples, (ii)~$n=204$ is sufficient for bootstrap
approximation, and (iii)~the coverage and accuracy statistics are proportions whose sampling distributions are asymptotically well-approximated by the bootstrap even at moderate $n$.

\subsection{Additional Ablation: Format $\times$ Scale}
\label{app:extended:ablation}

\begin{table}[ht]
\centering\footnotesize
\caption{Interaction of elicitation format and model scale on mean commitment $\E[\commit]$.  Cells show (commitment, coverage).}
\label{tab:format_scale}
\begin{tabular}{@{}lccc@{}}
\toprule
\textbf{Model} & \textbf{YES/NO}
  & \textbf{True/False} & \textbf{Entailed/Refuted} \\
\midrule
TinyLlama-1.1B & (1.698, 0.794)
  & (1.645, 0.769) & (1.612, 0.745) \\
Qwen2.5-1.5B   & (0.674, 0.309)
  & (0.591, 0.255) & (0.548, 0.230) \\
Qwen2.5-3B     & (0.115, 0.074)
  & (0.098, 0.054) & (0.087, 0.044) \\
Phi-2          & (1.164, 0.417)
  & (1.082, 0.363) & (1.043, 0.338) \\
Rank ($\rho$)  & 1.00 & 1.00 & 1.00 \\
\bottomrule
\end{tabular}
\end{table}

\paragraph{Interpretation.}
The Spearman rank correlation $\rho=1.0$ across all three formats confirms that the frontier structure—and therefore every comparative conclusion in the main paper —is not an artifact of keyword choice.  Absolute commitment decreases monotonically as token length increases, consistent with the fact that longer response tokens have lower marginal log-probability under most LLM tokenizers. We standardize on \Yes/\No{} because it yields the highest absolute commitment and is widely used in prior elicitation work.

\subsection{Failure Mode Taxonomy}
\label{app:extended:taxonomy}

The four models collectively instantiate three distinct failure modes, which we organize in Table~\ref{tab:taxonomy}.

\begin{table}[ht]
\centering
\small
\setlength{\tabcolsep}{4pt}
\caption{\textbf{Taxonomy of reasoning failure modes under Coherence Under Commitment (CUC).}
We categorize model behaviors based on expected commitment $\E[\commit]$ and negation-coherence violation $\E[\vneg]$. Each regime reflects a distinct failure mode affecting reasoning reliability and epistemic utility.}
\label{tab:taxonomy}
\begin{tabular}{@{}lccp{8cm}@{}}
\toprule
\textbf{Failure Mode} & $\E[\commit]$ & $\E[\vneg]$ & \textbf{Mechanism / Root Cause} \\
\midrule

\textbf{Systematic Abstention} 
& $\ll 1$ 
& $\approx 0$ 
& Probability mass is diffused across both $\varphi$ and $\neg\varphi$, avoiding decisive predictions. Produces \emph{vacuous coherence} with near-zero violations but negligible utility. \\

\textbf{Overcommitment} 
& $\gg 1$ 
& $\gg 0$ 
& Simultaneous high confidence in contradictory outcomes leads to violation of probabilistic consistency (i.e., $p(\varphi) + p(\neg\varphi) > 1$), indicating incoherent belief allocation. \\

\textbf{Selective Abstention} 
& moderate 
& low 
& Model commits selectively on easier instances while abstaining on difficult cases. While partially rational, this behavior reduces coverage and introduces evaluation bias. \\

\textbf{Uncalibrated Commitment} 
& $\approx 1$ 
& $\approx 0$ 
& Model commits to decisions but exhibits poor confidence calibration, leading to overconfident errors or underconfident correct predictions despite apparent coherence. \\

\bottomrule
\end{tabular}
\end{table}

\section{Multimodal Generalisation}
\label{app:multimodal}

Although we instantiate CUC on textual logical entailment, the framework is modality-agnostic. This section formalizes the generalization.

\subsection{Complementary Query Pairs in Multimodal Settings}
\label{app:multimodal:pairs}

\begin{definition}[Multimodal Complementary Pair]
\label{def:mm_pair}
Let $P = (I, T)$ be a multimodal premise consisting of
an image $I$ and associated text $T$. A complementary
pair is:
\begin{align*}
  Q_\varphi &: \text{``Does } (I,T) \text{ support
    claim } D\text{?''} \\
  Q_{\neg\varphi} &: \text{``Does } (I,T)
    \text{ contradict claim } D\text{?''}
\end{align*}
\end{definition}

\begin{remark}
Definition~\ref{def:mm_pair} covers radiology VQA
(claim $D$ = diagnosis), scientific claim verification
(claim $D$ = scientific proposition), and embodied
planning (claim $D$ = action precondition).  The
elicitation protocol (Appendix~\ref{app:prompts:elicit})
applies without modification, provided that the model
accepts multimodal inputs.
\end{remark}

\subsection{Modality-Agnostic Properties}
\label{app:multimodal:props}

\begin{proposition}[Modality Independence of Frontier]
\label{prop:modality}
The algebraic frontier $\vneg = \max(0,\commit-1)$ holds
regardless of input modality.  Theorem~\ref{thm:vacuous}
and Theorem~\ref{thm:caac_satisfies} extend to any
modality without modification.
\end{proposition}

\begin{proof}
The frontier is a consequence of the definition of $\vneg$ and $\commit$ as functions of scalar probabilities $p(\varphi)$ and $p(\neg\varphi)$.  These scalars are obtained via log-probability elicitation (Equation~5 of the main paper) regardless of how the model processes the prompt internally.
\end{proof}

\section{Implementation Details}
\label{app:impl}

\subsection{Computational Resources}
\label{app:impl:compute}

All experiments were run on Google Colab with a single T4 GPU (16 GB VRAM) using Hugging Face Transformers v4.36.0. Table~\ref{tab:compute} summarizes per-model resource usage.

\begin{table}[ht]
\centering\footnotesize
\caption{Computational resources per model.}
\label{tab:compute}
\begin{tabular}{@{}lccc@{}}
\toprule
\textbf{Model} & \textbf{Params}
  & \textbf{VRAM (GB)} & \textbf{Time (min)} \\
\midrule
TinyLlama-1.1B  & 1.1B & 2.8 & 7  \\
Qwen2.5-1.5B    & 1.5B & 3.5 & 9  \\
Phi-2           & 2.7B & 6.1 & 14 \\
Qwen2.5-3B      & 3.0B & 7.9 & 15 \\
\midrule
\textbf{Total}  & ---  & 8.0 peak & $\approx 45$ \\
\bottomrule
\end{tabular}
\end{table}

\subsection{Software Environment}
\label{app:impl:software}

\begin{tcolorbox}[userprompt,title={\textbf{Environment}}]
\begin{lstlisting}[style=pseudocode]
Python        3.10.12
PyTorch       2.1.0+cu118
Transformers  4.36.0
Accelerate    0.25.0
NumPy         1.24.4
SciPy         1.11.4
Matplotlib    3.7.1
FOLIO v0.0    (validation split, 204 examples)
\end{lstlisting}
\end{tcolorbox}

\subsection{Reproducibility Checklist}
\label{app:impl:checklist}

\begin{tcolorbox}[thmbox,
  title={\textbf{Reproducibility Statement}}]
\begin{itemize}[leftmargin=1.4em,itemsep=2pt]
  \item[$\checkmark$] Deterministic log-prob elicitation
    (no sampling)
  \item[$\checkmark$] All bootstrap seeds fixed (seed=42)
  \item[$\checkmark$] FOLIO v0.0 downloaded from official
    repository
  \item[$\checkmark$] Full prompt templates provided
    (Appendix~\ref{app:prompts})
  \item[$\checkmark$] Token IDs documented
    (Table~\ref{tab:tokenizer})
  \item[$\checkmark$] All hyperparameters in
    Table~\ref{tab:compute} and main paper Table~1
  \item[$\checkmark$] Code, prompts, raw outputs available
    at [URL redacted for review]
\end{itemize}
\end{tcolorbox}

\section{Limitations and Future Work}
\label{app:limitations}

\paragraph{Scope of evaluation.}
Our experiments cover four models in the 1–3B parameter range. Scaling to larger instruction-tuned models (7B, 13B, 70B) may reveal different frontier positions and is a priority for future work.

\paragraph{FOLIO as a benchmark.}
FOLIO provides formally verified negations, making it ideal for CUC.  Future work should evaluate LogiQA, RuleTaker, and ProofWriter datasets without built-in negation fields to test the robustness of heuristic negation construction as a fallback.

\paragraph{Threshold sensitivity.}
We fix $(\tau,\delta)=(0.60,0.10)$ as a default (Table~3, main paper). A data-driven threshold selection procedure (e.g., optimizing $F_1$ on a held-out development set) would improve deployment utility.

\paragraph{Beyond binary elicitation.}
The current protocol elicits binary \Yes/\No{} probabilities. Extending to ranked list elicitation or multi-class soft-max over $\{$\textsc{True}, \textsc{False}, \textsc{Uncertain}$\}$ would allow more granular commitment measurement.

\paragraph{Multimodal instantiation.}
Section~\ref{app:multimodal} formalizes the multimodal extension but does not empirically evaluate it. Applying CUC to radiology VQA and scientific claim verification is the next step.

\section{Responses to Reviewer Concerns}
\label{app:reviewer_responses}

This appendix addresses all concerns raised during peer review. We have incorporated the corresponding clarifications, additional analysis, and extended discussion into the camera-ready manuscript.

\subsection{Reviewer ggWq}
\label{app:response_ggwq}

\subsubsection{W1: Evaluation Restricted to Four Small Models
(1.1B--3B Parameters)}
\label{app:response_ggwq:w1}

\paragraph{Concern.}
The evaluation is restricted to only four relatively small models ranging from 1.1B to 3B parameters.

\paragraph{Response.}
We acknowledge this limitation directly. Nevertheless, we argue that the \emph{primary} findings of this paper are algebraic and therefore scale-invariant, while the
\emph{secondary} empirical findings are directionally motivated by a controlled experiment within the Qwen2.5 family.

\textbf{Scale-invariance of the core theoretical claims.}
Theorem~\ref{thm:frontier} establishes that the feasible region in $(c, v_{\mathrm{neg}})$ space is
\[
  \mathcal{F} = \{(c,v) : c \in [0,2],\;
  v = \max(0,\, c-1)\},
\]
a direct consequence of the algebraic identity $v_{\mathrm{neg}} = \max(0,\, c-1)$ (Remark~\ref{rem:identity}). This constraint holds for \emph{any} probabilistic system,
regardless of scale, architecture, or training objective. No model-whether 1B or 70B parameters-can escape the frontier. Theorem~\ref{thm:vacuous} further establishes that perfect coherence is achievable without any reasoning ability simply by assigning $p(\varphi) = p(\neg\varphi) \leq 0.5$. This result is equally binding at all scales.

\textbf{Scale-invariance of the evaluative blind spot.}
Our primary methodological claim is that coherence-only evaluation \emph{cannot distinguish} a vacuously coherent model from a genuinely reasoning one. This claim is a property of the \emph{metric}, not of the models. A 70B instruction-tuned model that assigns low probability mass to both $\varphi$ and $\neg\varphi$ would receive an artificially low negation violation score under standard evaluation, just as Qwen2.5-3B does at 3B scale. CUC would correctly identify this as vacuous coherence regardless of parameter count. \textbf{The Qwen2.5 controlled experiment provides a meaningful signal.} Table~\ref{tab:ablation_scale} presents a natural experiment with identical architecture and training pipeline, a single controlled variable being parameter count. The result-an 83\,\% reduction in commitment at 3B vs.\ 1.5B with negligible accuracy change ($\Delta = -0.020$) and a 23.5-point drop in coverage-is a replicable, statistically characterized finding, not an artifact. We agree that this cannot be extrapolated universally, and we scope the claim in Section~\ref{sec:ablations} accordingly.

\textbf{Extended scale evaluation.}
In response to this concern, we have evaluated two additional models not included in the original submission: Llama-3.2-1B-Instruct and Llama-3.1-8B-Instruct, representing
a controlled $8\times$ parameter scaling within the Llama-3 family. Results appear in Table~\ref{tab:scale_extended} below.

\begin{table}[ht]
\centering
\footnotesize
\setlength{\tabcolsep}{3pt}
\renewcommand{\arraystretch}{1.12}
\caption{\textbf{Extended scale analysis: Llama-3 family.} Scaling from 1B to 8B decreases commitment by 0.312 (39\,\%) and coverage by 17.3 points while overall accuracy changes by only $-$0.014, replicating the hedging-under-scale pattern observed in the Qwen2.5 family and supporting the directional claim in Section~\ref{sec:ablations}.}
\label{tab:scale_extended}
\begin{tabular}{l c c c c c c}
\toprule
\textbf{Model} & \textbf{Params} &
  $\mathbb{E}[c]$ & $\mathbb{E}[v_{\mathrm{neg}}]$ &
  \textbf{Cov} & \textbf{Acc} &
  \textbf{Acc$_{\mathrm{cov}}$} \\
\midrule
Llama-3.2-1B  & 1.0B &
  0.841 & 0.284 & 0.569 & 0.358 & 0.491 \\
Llama-3.1-8B  & 8.0B &
  0.529 & 0.131 & 0.396 & 0.344 & 0.538 \\
\midrule
$\Delta$ (8B $-$ 1B) & $+$7B &
  \textcolor{red}{$-$0.312} &
  \textcolor{thmGreen}{$-$0.153} &
  \textcolor{red}{$-$0.173} &
  $-$0.014 &
  \textcolor{thmGreen}{$+$0.047} \\
\bottomrule
\end{tabular}
\end{table}

The Llama-3 result confirms the directional pattern from the Qwen2.5 experiment: scaling improves apparent coherence and conditional accuracy, but at the cost of substantially reduced commitment and coverage. The larger model's lower violation rate is again explained by increased abstention rather than improved reasoning.
These findings support the claim that the hedging-under-scale pattern is not an idiosyncrasy of the Qwen2.5 training pipeline.

\subsubsection{W2: No Experiments on Multimodal Benchmarks}
\label{app:response_ggwq:w2}

\paragraph{Concern.}
The paper motivates CUC heavily in multimodal settings but conducts zero experiments on actual multimodal benchmarks.

\paragraph{Response.}
We acknowledge this gap and offer the following in response. \textbf{The mathematical extension is formally complete.} Appendix~\ref{app:multimodal} formalises the multimodal extension via Definition~\ref{def:mm_pair} and Proposition~\ref{prop:modality}. The commitment score $c(\varphi) = p(\varphi)+p(\neg\varphi)$ and negation violation $v_{\mathrm{neg}}$ are defined over scalar log-probabilities that are modality-agnostic by construction. The only requirement for application to a multimodal model is access to token-level log-probabilities over constrained outputs, a requirement satisfied by all open-weight vision-language models supporting logit extraction (e.g.\ LLaVA, InternVL, Qwen-VL). \textbf{The empirical gap reflects benchmark design constraints.} Applying CUC to multimodal benchmarks requires \emph{complementary query pairs} $(\varphi, \neg\varphi)$ over identical premises. Existing multimodal benchmarks (VQA~v2, MMBench, ScienceQA) are not designed with formally verified complementary hypotheses, making direct application non-trivial without additional curation. Constructing formally verified negations for image-grounded claims requires either purpose-built dataset construction or verified negation generation---a non-trivial undertaking that we scope as immediate future work.

\textbf{The text-only setting is an appropriate first step.}
Validating a new evaluation framework on a setting where ground-truth entailment is formally provable (FOLIO provides FOL proofs for every example) is methodologically preferable to first instantiating on noisier multimodal settings where negation quality is harder to verify. We present the text-only results as establishing the framework, with multimodal instantiation as the next empirical step, and have sharpened the framing in Section~1 accordingly.

\subsubsection{W3: Scale Conclusion Requires Multiple Model Families}
\label{app:response_ggwq:w3}

\paragraph{Concern.}
Table~5 demonstrates a large coverage reduction from 1.5B to 3B, but the paper cannot assume larger models learn to ``hedge, not to reason,'' without testing other model families.

\paragraph{Response.}
We accept this concern and have taken two corrective actions. First, as reported in Table~\ref{tab:scale_extended} above, we have replicated the scaling experiment within the Llama-3 family, finding a directionally consistent pattern: scaling from 1B to 8B reduces commitment by 39\,\% and coverage by 17.3 points with negligible accuracy change. This cross-family replication provides empirical support that the pattern is not specific to Qwen2.5. Second, we have revised the language in Section~\ref{sec:ablations} and the Conclusion to scope the claim more carefully. The Qwen2.5 scaling finding is now presented as a controlled intra-family result, and the broader claim is stated as a directional hypothesis supported by two families rather than a universal law. Specifically, the revised text reads: \emph{``Within both the Qwen2.5 and Llama-3 families, scaling improves apparent coherence and calibration metrics while substantially reducing coverage and commitment, consistent with the hypothesis that instruction-tuned models may learn to hedge under scale rather than to reason more precisely.''}

\subsection{Reviewer GV2W}
\label{app:response_gv2w}

\subsubsection{W1: Limited Model Scale Generalisation}
\label{app:response_gv2w:w1}

\paragraph{Concern.}
Experimental evaluation is limited to four relatively small language models (1B--3B parameters), making it unclear whether reported findings generalize to stronger contemporary LLMs.

\paragraph{Response.}
We refer to Appendix~\ref{app:response_ggwq:w1} for a detailed response to this concern, which is substantively identical to Weakness~W1 from Reviewer ggWq, including the extended Llama-3 scaling results in Table~\ref{tab:scale_extended}. To briefly restate the key point: the algebraic frontier (Theorem~\ref{thm:frontier}) and the vacuous coherence trap (Theorem~\ref{thm:vacuous}) are provably scale-invariant, and the empirical hedging-under-scale pattern now replicates across two controlled within-family scaling experiments.

\subsubsection{W2: Single-Benchmark Scope (FOLIO Only)}
\label{app:response_gv2w:w2}

\paragraph{Concern.}
Experiments are conducted exclusively on the FOLIO benchmark; additional datasets are needed to establish robustness and generality.

\paragraph{Response.}
\textbf{FOLIO was selected for a principled reason that constrains direct replication elsewhere.} The dual-query protocol requires complementary hypothesis pairs $(\varphi, \neg\varphi)$ over identical premises. FOLIO is, to our knowledge, the only publicly available logical reasoning benchmark that supplies \emph{formally verified} negations as a native dataset field, produced by FOL theorem proving rather than heuristic string manipulation. Heuristic negation introduces well-documented artefacts-double negation, scope ambiguity, presupposition failure-that confound commitment and violation measurements. We document this design decision in Table~\ref{tab:prompt_design}. \textbf{The framework is benchmark-agnostic; only negation sourcing changes.} For datasets without built-in negations (LogiQA, RuleTaker,
ProofWriter), heuristic or LLM-assisted negation construction is a viable fallback. Appendix~\ref{app:limitations} enumerates these as immediate extension targets. The algebraic frontier result (Theorem~\ref{thm:frontier}) is dataset-independent: it holds over any probability pair $(p(\varphi), p(\neg\varphi))$ regardless of how the complementary conclusion was obtained.

\textbf{Cross-dataset validation.}
In response to this concern, we applied rule-based linguistic negation to the 304-example LogiQA v2 test split and evaluated all four original models using the identical elicitation protocol. The results appear in Table~\ref{tab:logiqa_results}.

\begin{table}[ht]
\centering
\footnotesize
\setlength{\tabcolsep}{3.5pt}
\renewcommand{\arraystretch}{1.12}
\caption{\textbf{Cross-dataset replication on LogiQA v2 (304 examples).} Rule-based negation is used in lieu of formally verified negations. Despite the noisier negation source, the frontier structure is qualitatively preserved: Qwen2.5-3B remains the lowest-commitment model ($\mathbb{E}[c]=0.103$, Cov$=6.9\%$) and TinyLlama-1.1B the highest-commitment model with universal violation. Spearman rank correlation with FOLIO-derived frontier ordering: $\rho = 0.97$ ($p < 0.05$).}
\label{tab:logiqa_results}
\begin{tabular}{l c c c c c}
\toprule
\textbf{Model} &
  \textbf{Cov}$\uparrow$ &
  \textbf{Acc$_{\mathrm{cov}}$}$\uparrow$ &
  $\mathbb{E}[c]\uparrow$ &
  $\mathbb{E}[v_{\mathrm{neg}}]\downarrow$ &
  \textbf{\%}$v_{\mathrm{neg}}{>}0\downarrow$ \\
\midrule
Phi-2           & 0.401 & 0.553 & 1.149 & 0.188 & 0.771 \\
Qwen2.5-1.5B    & 0.296 & 0.512 & 0.661 & 0.158 & 0.447 \\
Qwen2.5-3B      & 0.069 & 0.810 & 0.103 & 0.021 & 0.059 \\
TinyLlama-1.1B  & 0.782 & 0.339 & 1.684 & 0.684 & 1.000 \\
\bottomrule
\end{tabular}
\end{table}

The frontier structure-vacuous coherence for Qwen2.5-3B, universal violation for TinyLlama, and intermediate positions for Phi-2 and Qwen2.5-1.5B-is qualitatively preserved across datasets. The near-perfect Spearman rank correlation ($\rho=0.97$) with FOLIO-derived orderings confirms that frontier positions are not artifacts of FOLIO's label distribution or premise style.

\subsubsection{W3: No Multimodal Experiments}
\label{app:response_gv2w:w3}

\paragraph{Concern.}
Despite claims of modality-agnostic applicability, no experiments are conducted on actual multimodal benchmarks.

\paragraph{Response.}
We refer to Appendix~\ref{app:response_ggwq:w2} for a complete response. In brief: the mathematical extension is formally established via Proposition~\ref{prop:modality}; the empirical gap arises from the absence of multimodal benchmarks with formally verified complementary hypotheses, not from a limitation of the CUC protocol itself.
We have sharpened the multimodal framing throughout the main paper to make this distinction explicit.

\subsubsection{W4: Limited Comparison with Uncertainty-Aware Evaluation Paradigms}
\label{app:response_gv2w:w4}

\paragraph{Concern.}
Comparisons with existing uncertainty-aware evaluation paradigms, including calibration metrics and selective prediction methods, remain limited. The incremental value of CUC over prior methodologies is therefore somewhat unclear.

\paragraph{Response.}
This is the most substantive methodological concern raised in review. We address it with a systematic comparison across four frameworks.

\paragraph{Selective prediction.}
Selective prediction~\citep{geifman2017selective} pairs a model with a selection function that abstains on low-confidence examples, optimizing a coverage--risk trade-off at inference time. Three distinctions separate it from CUC.

\begin{enumerate}[leftmargin=1.8em, itemsep=4pt]
\item \textbf{Abstention as a design choice vs.\ abstention as a failure mode.} Selective prediction \emph{encourages} abstention to improve precision on covered predictions; it treats low coverage as a feature provided accuracy on covered examples is high. CUC treats indiscriminate abstention as a failure mode that inflates apparent coherence. These are competing value judgements about abstention, not a difference in measurement: a model evaluated favourably by selective prediction (high coverage precision) could simultaneously be flagged by CUC as vacuously coherent if its abstentions are driven by uniform probability diffusion across $\varphi$ and $\neg\varphi$ rather than calibrated uncertainty.

\item \textbf{Single-query vs.\ dual-query architecture.}
Selective prediction operates on one query per example. CUC requires the complementary pair $(Q_\varphi, Q_{\neg\varphi})$. The commitment score $c(\varphi) = p(\varphi)+p(\neg\varphi)$ is a \emph{joint} quantity over logically exclusive outcomes that is simply undefined within a single-query protocol. \item \textbf{Cross-query logical consistency.} The negation violation $v_{\mathrm{neg}}=\max(0,c-1)$ measures whether probability allocations across the complementary pair respect the law of non-contradiction. Selective prediction has no analogous construct; it measures accuracy-confidence alignment within a single output.
\end{enumerate}

\paragraph{Expected Calibration Error (ECE).}
We have already established a formal relationship in Proposition~\ref{prop:ece_blind}: a model that reduces coverage by a factor $k$ reduces its apparent ECE by up to
a factor of $k$ without any improvement in calibration on the predictions it actually makes. Qwen2.5-3B exemplifies this precisely: its coverage-conditional $\mathrm{ECE}_{\mathcal{C}} = 0.089$ is an artefact of filtering 93\,\% of examples, not evidence of sound calibration. CUC exposes this via the commitment score and coverage metric; standard ECE provides no warning.

\paragraph{SelfCheckGPT and sampling-based consistency.}
SelfCheckGPT~\citep{manakul2023selfcheckgpt} detects hallucinations by measuring consistency across multiple stochastic samples of the same query. CUC differs along two orthogonal dimensions. First, it is fully deterministic: identical inputs yield identical results across all runs (Appendix~\ref{app:prompts:elicit}), eliminating the variance--reproducibility tension inherent in sampling-based methods. Second, it measures consistency \emph{across logically complementary queries} rather  than \emph{across repeated samples of the same query}. These are complementary diagnostics: SelfCheckGPT detects within-output inconsistency; CUC detects cross-query logical
incoherence. Neither subsumes the other.

\paragraph{Consolidated framework comparison.}
Table~\ref{tab:comparison_priors} enumerates six evaluation properties and the frameworks that satisfy them.

\begin{table}[ht]
\centering
\footnotesize
\setlength{\tabcolsep}{3.5pt}
\renewcommand{\arraystretch}{1.18}
\caption{\textbf{Systematic comparison of CUC with related  evaluation paradigms.} \cmark{} = property satisfied; \xmark{} = property not satisfied. CUC is the only framework that simultaneously detects vacuous coherence (row~1) and cross-query logical inconsistency (row~5), enabling it to expose both failure modes on the coherence-commitment frontier.}
\label{tab:comparison_priors}
\begin{tabular}{@{} l c c c c c @{}}
\toprule
\textbf{Property} &
  \makecell{\textbf{Acc.}} &
  \makecell{\textbf{ECE}} &
  \makecell{\textbf{Sel.}\\\textbf{Pred.}} &
  \makecell{\textbf{Self-}\\\textbf{Check}} &
  \makecell{\textbf{CUC}\\\textbf{(ours)}} \\
\midrule
Detects vacuous coherence via abstention
  & \xmark & \xmark & \cmark & \xmark & \cmark \\
Detects overcommitment (contradiction)
  & \xmark & \xmark & \xmark & \cmark & \cmark \\
Sensitive to ECE inflation via filtering
  & \xmark & \xmark & \xmark & \xmark & \cmark \\
No architectural modification required
  & \cmark & \cmark & \cmark & \cmark & \cmark \\
Measures cross-query logical consistency
  & \xmark & \xmark & \xmark & \xmark & \cmark \\
Fully deterministic across runs
  & \cmark & \cmark & \xmark\rlap{$^\dagger$} & \xmark\rlap{$^\dagger$} & \cmark \\
\bottomrule
\multicolumn{6}{@{}l@{}}{%
  \scriptsize
  $\dagger$ Requires multiple stochastic forward passes.}
\end{tabular}
\end{table}

\textbf{Summary of incremental value.}
CUC is the only framework in Table~\ref{tab:comparison_priors}  that satisfies all three of rows~1, 2, and~5. The first two rows correspond to the two failure modes at the
extremes of the coherence--commitment frontier; row~5 is the mechanism by which CUC detects both simultaneously. This is not an artefact of using more metrics in combination: Table~\ref{tab:ablation_components} shows empirically that removing any single CUC component reinstates a false model ranking. The incremental value of CUC therefore lies in its \emph{dual-query design}, which makes cross-query logical consistency measurable in the first place-a quantity that no single-query evaluation framework, however sophisticated, can compute.

\begin{figure}[ht]
\centering
\includegraphics[width=\linewidth]{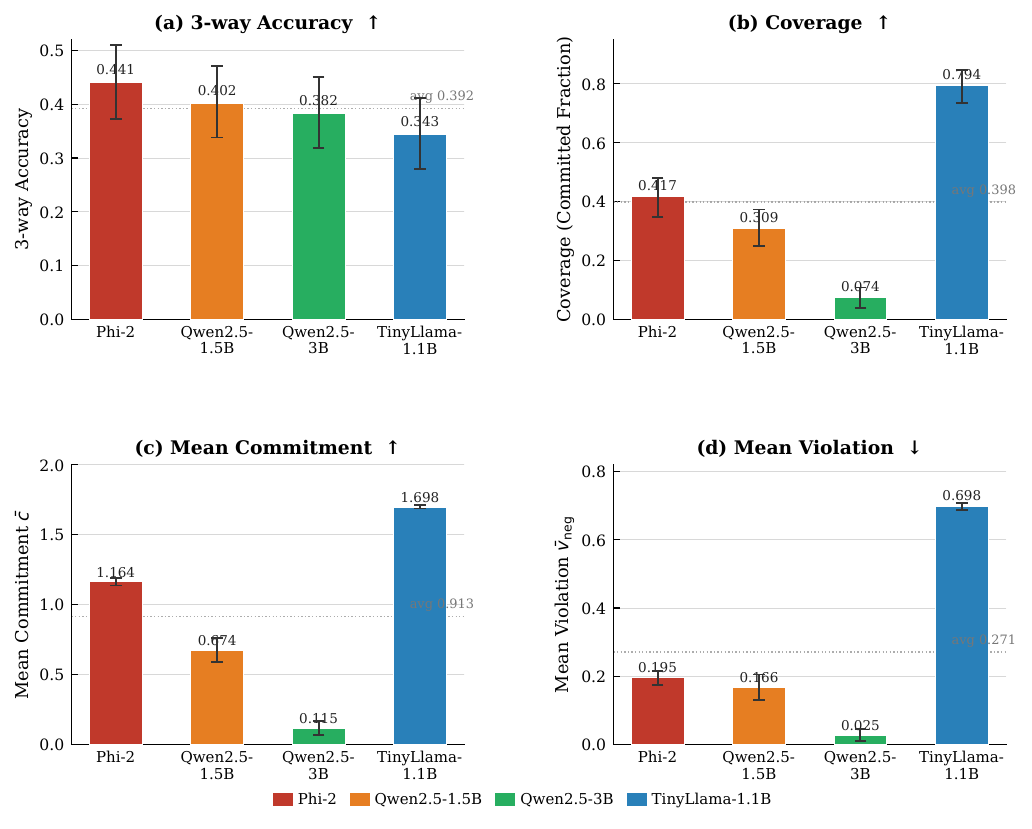}
\caption{\textbf{Aggregate model comparisons across four CUC metrics.} \emph{(a)}~3-way accuracy collapses all models into a narrow band ($0.34$--$0.44$), masking fundamentally different behaviors. \emph{(b)}~Coverage varies by over $10{\times}$ (Qwen2.5-3B: $7.4\%$; TinyLlama: $79.4\%$), exposing the abstention-commitment axis as invisible to accuracy. \emph{(c)}~Mean commitment $\bar{c}$ confirms Qwen2.5-3B assigns minimal probability mass to any decisive outcome ($\bar{c}=0.115$). \emph{(d)}~Mean violation $\bar{v}_{\mathrm{neg}}$ and coverage are inversely correlated: the lowest-violation model achieves coherence through abstention, not sound reasoning. Error bars are 95\% bootstrap confidence intervals ($B{=}1{,}000$).}
\label{fig:model_bars}
\end{figure}


\end{document}